\documentclass{article}

\usepackage{microtype}
\usepackage{graphicx}
\usepackage{subfigure}
\usepackage{booktabs} % for professional tables

\usepackage{hyperref}

\usepackage[accepted]{icml2024}

\usepackage{amsmath}
\usepackage{amssymb}
\usepackage{mathtools}
\usepackage{amsthm}

\usepackage{adjustbox}
\usepackage{tabularx} % for tables with adjustable column width

\usepackage[capitalize,noabbrev]{cleveref}

\theoremstyle{plain}
\newtheorem{theorem}{Theorem}[section]

\newtheorem{lemma}[theorem]{Lemma}

\theoremstyle{definition}

\theoremstyle{remark}

\usepackage[textsize=tiny]{todonotes}

\icmltitlerunning{The Max-Min Formulation of Multi-Objective Reinforcement Learning: From Theory to a Model-Free Algorithm}

\begin{document}

\twocolumn[
\icmltitle{The Max-Min Formulation of Multi-Objective Reinforcement Learning:\\ From Theory to a Model-Free Algorithm}



\icmlsetsymbol{equal}{*}

\begin{icmlauthorlist}
\icmlauthor{Giseung Park}{kor}
\icmlauthor{Woohyeon Byeon}{kor}
\icmlauthor{Seongmin Kim}{kor}
\icmlauthor{Elad Havakuk}{isr}
\icmlauthor{Amir Leshem}{isr}
\icmlauthor{Youngchul Sung}{kor}
\end{icmlauthorlist}

\icmlaffiliation{kor}{School of Electical Engineering, Korea Advanced Institute of Science \& Technology, Daejeon 34141, Republic of Korea}
\icmlaffiliation{isr}{Faculty of Engineering, Bar-Ilan University, Ramat Gan 52900, Israel}

\icmlcorrespondingauthor{Youngchul Sung}{ycsung@kaist.ac.kr}

\icmlkeywords{Machine Learning, ICML}

\vskip 0.3in
]



\printAffiliationsAndNotice{}  % leave blank if no need to mention equal contribution

\begin{abstract}
In this paper, we consider multi-objective reinforcement learning, which arises in many real-world problems with multiple optimization goals. We approach the problem with a max-min framework focusing on fairness among the multiple goals and develop a relevant theory and a practical model-free algorithm under the max-min framework. The developed theory provides a theoretical advance in multi-objective reinforcement learning, and the proposed algorithm demonstrates a notable performance improvement over existing baseline methods.
\end{abstract}

\section{Introduction and Motivation}
\label{intro}

Reinforcement Learning (RL) is a powerful machine learning paradigm, focusing on training an agent to make sequential decisions   by interacting with an environment. RL algorithms learn to maximize the cumulative  reward sum  through a trial-and-error process, enabling the agent to adapt and improve its decision-making strategy over time.   Recently, the field of Multi-Objective Reinforcement Learning (MORL) has received increasing attention from the RL community since many practical control problems are formulated as multi-objective optimization.  For example, consider a scenario where an autonomous vehicle must balance the competing goals of reaching its destination swiftly while ensuring passenger safety.  MORL extends traditional RL to address such scenarios in which  multiple, often conflicting, objectives need to be optimized simultaneously \cite{roijers13survey,hayes22survey}.

Formally, a multi-objective Markov decision process (MOMDP) is defined as $<\hspace{-0.3em}S, A, P, \mu_0, r,   \gamma \hspace{-0.3em}>$, where $S$ and $A$ represent the sets of states and actions, respectively, $P: {S \times A} \rightarrow \mathcal{P}(S)$ is the transition probability function where $\mathcal{P}(S)$ is the space of probability distributions over $S$,  $\mu_0: S \rightarrow [0,1]$ represents the initial distribution of  states, and $\gamma \in [0,1)$ is the discount factor.  The key difference from the conventional RL is that ${r}: S \times A \rightarrow \mathbb{R}^K$  is a \textbf{vector-valued} reward function with  $K \geq 2$. 
At each timestep $t$, the agent draws an action $a_t \in A$ based on current state $s_t \in S$ from its policy $\pi: S \rightarrow \mathcal{P}(A)$ which is a probability distribution over $A$. Then, the environment state makes a transition from the current state $s_t$ to the next state $s_{t+1} \in S$  with probability $P(s_{t+1} | s_t, a_t)$, and the agent receives a vector-valued reward ${r}_t = [r_t^{(1)}, \cdots, r_t^{(K)}]^T =  {r}(s_t, a_t) \in \mathbb{R}^K$, where $[\cdot]^T$ denotes the transpose operation.  Let $J(\pi) := \mathbb{E}_\pi \left[ \sum_{t=0}^\infty \gamma^t {r}_t \right] = [J_1(\pi), \cdots, J_K(\pi)]^T \in \mathbb{R}^K$. 

\begin{figure}[!t] \label{fig:pareto_boundary}
% \vskip 0.2in
\begin{center}
\centerline{\includegraphics[width=0.75\columnwidth]{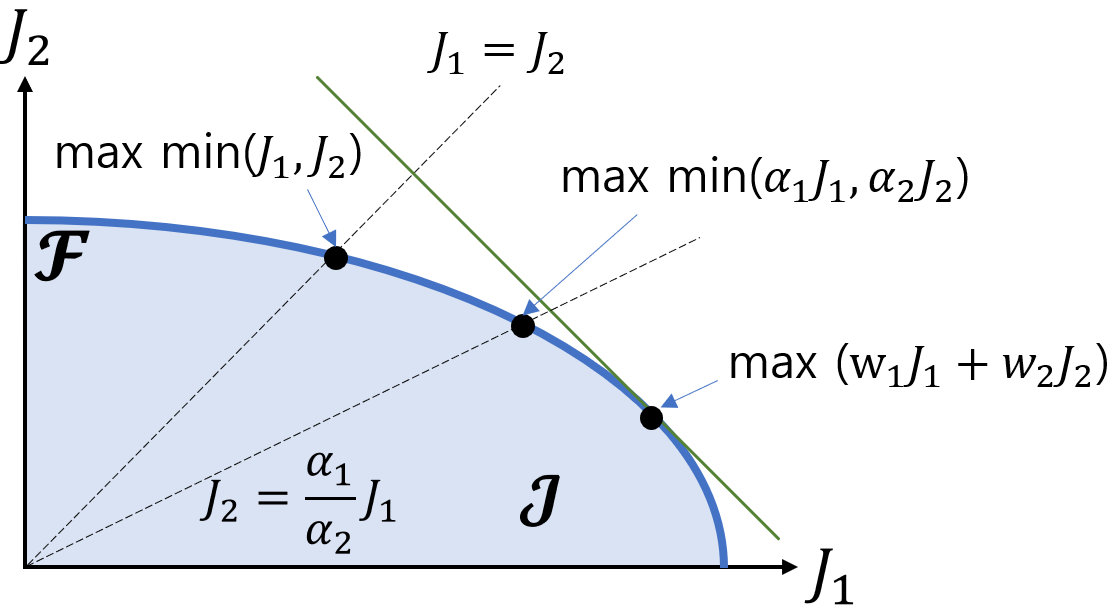}}
\caption{Achievable return region and Pareto boundary ($K=2$): weighted sum versus max-min approaches (Due to the equalizer rule \cite{zehavi2013weighted}, the max-min solution occurs on the line $J_1=J_2$.  On the other hand, the maximum sum $J_1+J_2$ occurs on the tangent line with slope -1. Controlling the ratio $\alpha_1/\alpha_2$, we can recover all points on the Pareto boundary by the max-min approach.) }
\label{pareto_boundary}
\end{center}
\vskip -0.4in
\end{figure}
 
In the standard RL case ($K=1$), the agent's goal is to find an optimal policy $\pi^* = \arg \max_{\pi \in \Pi } J(\pi)$, where $\Pi$ is the set of policies. On the other hand, the primary goal of MORL is to find a policy whose expected cumulative return vector lies on the {\em Pareto boundary} $\mathcal{F}$ of all achievable return tuples $\mathcal{J} = \{ (J_1(\pi),\cdots,J_K(\pi)), \forall \pi \in \Pi\}$, which is defined as the set of the return tuples in $\mathcal{J}$ for which any one $J_i$ cannot be increased without decreasing 
 $J_j$ for some $j\ne i$. Fig. \ref{fig:pareto_boundary} depicts $\mathcal{J}$ and $\mathcal{F}$.
 
A standard way to find such a policy is the utility-based approach \cite{roijers13survey}, which is formulated as the following policy optimization: $\pi^* = \arg \max_{\pi} f(J(\pi))$ such that $J(\pi^*) \in \mathcal{F}$, where the scalarization function $f: \mathbb{R}^K \rightarrow \mathbb{R}$ is non-decreasing function. A common example is the  weighted sum $f(J(\pi)) = \sum_{k=1}^K w_k J_k(\pi)$ with $\sum_k w_k =1, ~w_k \ge 0,\forall k$. In the case of weighted sum, by sweeping the weights $\{w_k\}$, different points of the Pareto boundary $\mathcal{F}$ can be found, as seen in Fig. \ref{fig:pareto_boundary}. However, the main disadvantage of the weighted sum approach is that we do not have a direct control of individual $J_1, \cdots, J_K$ and we may have an unfair case in which a particular $J_i$ is very small even if the weighted sum is maximized with seemingly-proper weights \cite{hayes22survey}.  Such an event depends on the shape of the achievable return region but we do not know the shape beforehand.

In order to address fairness across different dimensions of return $J_1,\cdots, J_K$, we adopt  the egalitarian welfare function $f=\min$ and explicitly formulate the max-min MORL in this paper. Unlike the widely-used weighted sum approach, the max-min approach ensures fairness in optimizing multiple objectives and is widely used in various practical applications such as resource allocation and multi-agent learning. Furthermore, by incorporating weights, the weighted max-min optimization recovers the convex coverage of the Pareto boundary \cite{zehavi2013weighted}, as seen in Fig. \ref{fig:pareto_boundary}. (Please see Appendix \ref{append:wide_use_maxmin} for details on applications and Pareto boundary recovery.)

Our contributions are summarized below:  

1) We developed a relevant theory for the max-min MORL based on a linear programming approach to RL. We show that the max-min MORL can be formulated as a joint optimization of the value function and a set of weights.
    
2) We introduced an entropy-regularized convex optimization approach to the max-min MORL which produces the max-min policy without ambiguity.
    
3) We proposed a practical model-free MORL algorithm that outperforms baseline methods in the max-min sense for the considered multi-objective tasks.

\section{Value Iteration as Linear Programming}
\label{background}

In standard RL with a scalar reward function $r$ with $K=1$, denoted as $r(s,a) \in \mathbb{R}$, $\forall (s,a) $, the Bellman optimality equation for the optimal value function $v^*$ is defined as 
\begin{equation}  \label{eq:VsBOE}
v^*(s) = \max_{a \in A} \left[ r(s,a) + \gamma \sum_{s' \in S} P(s'|s,a) v^*(s') \right], \forall s.
\end{equation}
Value iteration employs the Bellman optimality operator $T^*$ to compute $v^*$, which is expressed as $T^* v(s) := \max_{a \in A} \left[ r(s,a) + \gamma \sum_{s' \in S} P(s'|s,a) v(s') \right], \forall s$.

Interestingly, the optimal value function $v^*$ can be obtained by solving the following Linear Programming (LP) \cite{Puterman_2005}:
\begin{equation} \label{eq:equiv_lp_1}
    \min_{v} \sum_{s \in S} \mu_0(s) v(s)  
\end{equation}
subject to
\begin{equation} \label{eq:equiv_lp_2}
     v(s) \geq r(s,a) + \gamma \sum_{s' \in S}P(s'|s,a)v(s'), ~~\forall (s,a) .
\end{equation}
The LP seeks to minimize a linear combination of state values while satisfying constraints that mirror the Bellman optimality equation. When the dual transform of the LP  \eqref{eq:equiv_lp_1} and \eqref{eq:equiv_lp_2} is taken, it yields the following dual form:
\begin{equation} \label{eq:dual_lp_1}
    \max_{d}  \sum_{s,a } r(s,a) d(s,a)
\end{equation}
subject to
\begin{equation} \label{eq:dual_lp_2}
    \sum_{a' \in A} d(s',a') = \mu_0(s') + \gamma \sum_{s,a } P(s' | s,a) d(s,a), ~ \forall s'.
\end{equation}
\begin{equation} \label{eq:dual_lp_3}
    d(s,a) \geq 0, ~ \forall (s,a) .
\end{equation}
Note that \eqref{eq:dual_lp_2} is the balance equation for  the  (unnormalized) state-action visitation frequency \cite{sutton2018reinforcement}.  Hence, the dual variable $d(s,a)$  satisfying  \eqref{eq:dual_lp_2} and \eqref{eq:dual_lp_3} is equivalent to  an (unnormalized) state-action visitation frequency.   This frequency or distribution is independent of the rewards $r(s,a)$ and is expressed as $d(s,a) = \sum_{s'} \mu_0(s') \sum_{t=0}^\infty \gamma^t \text{Pr}( S_t = s, A_t = a |S_0 = s', \pi^d)$, where 
\begin{equation}  \label{eq:piddstar}
\pi^d(a|s) = \frac{d(s,a)}{\sum_{a'}d(s,a')}
\end{equation}
is the  stationary Markov policy induced by $d$ \cite{Puterman_2005}. Due to one-to-one mapping between $d$ and corresponding policy $\pi^d$, an optimal policy can be obtained from an optimal distribution $d^*$ from   (\ref{eq:dual_lp_1}, \ref{eq:dual_lp_2},  \ref{eq:dual_lp_3})
\cite{Puterman_2005}.

\section{Max-Min MORL with LP Formulation}
\label{method}

\subsection{Max-Min MORL Formulation}

The main problem considered in  this paper is the following   max-min MORL problem:
\begin{equation}  \label{eq:originalProblem}
    \max_{\pi \in \Pi}  \min_{1 \leq k \leq K} J_k(\pi), ~~~\mbox{where}~~ K \geq 2.
\end{equation}
Due to the non-linearity of the min operation, the above optimization problem cannot be solved directly \cite{roijers13survey} like the weighted sum case in which  we simply apply the conventional scalar reward RL methods to the weighted sum reward $\sum_{k=1}^K w_k r^{(k)}$.  To circumvent the difficulty in handling the min operation, we exploit the state-action visitation frequency $d(s,a)$ in (\ref{eq:dual_lp_2}) and (\ref{eq:dual_lp_3}). Note that this frequency is independent of the reward function and represents the relative frequency (or stationary distribution) of $(s,a)$ in the trajectory. 
Then, the  max-min 
 problem can equivalently be expressed as 
\begin{equation} \label{eq:maxmin_obj}
\textbf{P0}:     ~~~\max_{d} \min_{1 \leq k \leq K}  \sum_{s,a} d(s,a) r^{(k)}(s,a) 
\end{equation}
\begin{equation} \label{eq:maxmin_transition_eq}
    \sum_{a'} d(s',a') = \mu_0(s') + \gamma \sum_{s,a} P(s' | s,a) d(s,a) ~~ \forall s'
\end{equation}
\begin{equation} \label{eq:maxmin_transition_pos}
    d(s,a) \geq 0, ~~ \forall (s,a). 
\end{equation}
This formulation is valid due to the existence of an optimal stationary policy for any non-decreasing scalarization function \cite{roijers13survey}. 

The problem \textbf{P0} can be reformulated as an LP named \textbf{P0-LP} by using a slack variable to handle the min operation (please see Appendix \ref{append:a_duality}).  By solving the LP \textbf{P0-LP} equivalent to \textbf{P0}, we obtain $d^*(s,a)$ and an optimal policy from \eqref{eq:piddstar}.
However, solving \textbf{P0-LP}  requires prior knowledge of $r^{(k)}(s,a)$ and $P(s'|s,a)$. The main question of this paper is ``Can we find the max-min solution in a \textbf{model-free} manner without knowing the model $r^k(s,a)$ and $P(s'|s,a)$?"  In the following, we develop a relevant theory and propose a practical model-free max-min MORL algorithm.

To achieve this, we first convert \textbf{P0} into an LP  \textbf{P1}, which is the dual form of   \textbf{P0-LP} (for detailed derivation, please refer to Appendix \ref{append:a_duality}):

\begin{equation} \label{eq:target_LP_1}
    \textbf{P1}: ~ \min_{w \in \Delta^K, v} \sum_s \mu_0(s) v(s)
\end{equation}
\begin{equation} \label{eq:target_LP_2}
     v(s) \geq \sum_{k=1}^K w_k r^{(k)}(s,a) + \gamma \sum_{s'}P(s'|s,a)v(s'), ~ \forall (s,a) 
\end{equation}
where $\Delta^K := \{ w =[w_1,\cdots,w_K]^T\in \mathbb{R}^K | \sum_{k=1}^K w_k = 1; ~ w_k \geq 0, ~~ \forall k  \}$ is the $(K-1)$-simplex. Note that when $K=1$, \textbf{P1} simplifies to the LP (\ref{eq:equiv_lp_1}) and (\ref{eq:equiv_lp_2}) equivalent to value iteration in standard RL. Note also that $w$ does not appear in the optimization cost in (\ref{eq:target_LP_1}), but  appears  in the constraints in (\ref{eq:target_LP_2}). Hence, $w$ affects the feasible set of $v(s)$ and thus affects the cost through the feasible set of $v(s)$.   If we fix the weight vector $w \in \Delta^K$, the solution to \eqref{eq:target_LP_1} and \eqref{eq:target_LP_2} for $v$ corresponds to the result of value iteration using the scalarized reward function $\sum_{k=1}^K w_k r^{(k)}(s,a)$. Therefore, the feasible set of \textbf{P1} is non-empty. Let $(w_{LP}^{op},v_{LP}^{op})$ be the solution of \textbf{P1}.

\subsection{Equivalent Convex Optimization} \label{subsec:P1_to_P2}

If we insert the weight $w = w^{op}_{LP}$ in \textbf{P1}, the corresponding LP is reformulated as the following equivalent value iteration by  the relationship between (\ref{eq:VsBOE}) and (\ref{eq:equiv_lp_1}, \ref{eq:equiv_lp_2}): 
\begin{equation} \label{eq:equivalence_observation}
    v(s) \hspace{-0.1em}=\hspace{-0.1em} \max_a [\sum_{k=1}^K  \hspace{-0.1em}w^{op}_{LP, k} r^{(k)}\hspace{-0.1em}(s,a) \hspace{-0.1em}+\hspace{-0.1em} \gamma \hspace{-0.1em}\sum_{s'}\hspace{-0.1em}P(s'|s,a)v(s')],  \forall s
\end{equation}
where $v^{op}_{LP}$ should be the solution. Therefore, $v^{op}_{LP}$ is the unique fixed point attained by value iteration with the optimally scalarized reward function $\sum_k w^{op}_{LP, k} r^{(k)}(s,a)$.

Inspired by this observation, we define the following Bellman optimality operator $T^*_{w}$ for  a given  weight vector $w \in \mathbb{R}^K$ as 
\begin{equation} \label{eq:w_bellman_optimality_operator}
    T^*_{w} v(s) \hspace{-0.1em}:=\hspace{-0.1em} \max_a\hspace{-0.1em} \left[ \hspace{-0.1em}\sum_{k=1}^K \hspace{-0.1em}w_k r^{(k)}(s,a) \hspace{-0.1em}+\hspace{-0.1em} \gamma \sum_{s'} P(s'|s,a) v(s') \right]
\end{equation}
$\forall s.$
Let  $v^{*}_{w}$, which is a function of $w$, be the unique fixed point of the mapping $T^*_{w}$.

We now consider the following problem: 
\begin{equation} \label{eq:equivalent_convex}
    \textbf{P2}: ~ \min_{w \in \Delta^K} \mathcal{L}(w) = \min_{w \in \Delta^K}\sum_s \mu_0(s) v^{*}_{w}(s)
\end{equation}
where $\Delta^K$ is the $(K-1)$-simplex. In Theorem \ref{thm_1}, we show that \textbf{P2} is a convex optimization. Let $w^*$ be an optimal solution of \eqref{eq:equivalent_convex} whose existence is guaranteed by Theorem \ref{thm_1} and the fact that $\mathcal{L}(w)$ is continuous on $\Delta^K$\cite{rockafellar1997convex}. In Theorem \ref{thm_2}, we show that \textbf{P1} and \textbf{P2} have the same optimal value, and $(w^*, v^{*}_{w^*})$ is an optimal solution of \textbf{P1}. These steps are the milestones for devising our model-free algorithm in Section \ref{sec:model_free}.

\begin{theorem}\label{thm_1}
    For each $s$, $v^{*}_{w}(s)$ is a convex function in $w \in \mathbb{R}^K$. Consequently, the objective function $\mathcal{L}(w) = \sum_s \mu_0(s) v^{*}_{w}(s)$ is also convex in $w \in \mathbb{R}^K$.
\end{theorem}

\textit{Proof sketch.} 
For $0 \leq \lambda \leq 1$ and $w^1, w^2 \in \mathbb{R}^K$, let $\bar{w}_{\lambda} := \lambda w^1 + (1-\lambda) w^2$, and set $v:S\to\mathbb{R}$ arbitrary. We show that for any positive integer $p \geq 1$,
\begin{equation}
    (T^*_{\bar{w}_{\lambda}})^p v \leq \lambda (T^*_{w^1})^p v + (1-\lambda)(T^*_{w^2})^p v.
\end{equation}
(Please see  Appendix \ref{append:convexity} for  full derivation.) By letting $p \rightarrow \infty$, i.e., applying $T^*_{\bar{w}_{\lambda}}$ infinitely many times,  we obtain $v^*_{\bar{w}_{\lambda}}(s) \leq \lambda v^*_{w^1}(s) + (1-\lambda) v^*_{w^2}(s), ~ \forall s$. Then the objective function $\mathcal{L}(w) = \sum_s \mu_0(s) v^{*}_{w}(s)$ is also convex for $w \in \mathbb{R}^K$.

Since $\mathcal{L}(w)$ is convex,  $\mathcal{L}(w)$ is continuous with respect to  $w$ \cite{rockafellar1997convex} and the minimum value exists on $\Delta^K$.

\begin{theorem}\label{thm_2}
    Let $p^{op}_{LP} = \sum_{s} \mu_0(s) v^{op}_{LP}(s)$ be the  value of an optimal solution $(w^{op}_{LP}, v^{op}_{LP})$  of \textbf{P1} in \eqref{eq:target_LP_1} and \eqref{eq:target_LP_2}. 
Let $w^*$ be an optimal solution of \textbf{P2} in   \eqref{eq:equivalent_convex}. Then, \textbf{P1} and \textbf{P2} have the same optimal value (i.e., $p^{op}_{LP} = \mathcal{L}(w^*)$). In addition, $(w^*, v^{*}_{w^*})$ is an optimal solution of \textbf{P1} and $w^{op}_{LP}$ is an optimal solution of \textbf{P2}.
\end{theorem}

\begin{proof}
    The optimal value of \textbf{P2} is $\mathcal{L}(w^*) = \sum_s \mu_0(s) v^{*}_{w^*}(s)$. 
    \begin{enumerate}
        \item From \eqref{eq:equivalence_observation} we have $v^{op}_{LP} = T^*_{w^{op}_{LP}} v^{op}_{LP} =  v^*_{w^{op}_{LP}}$ with $w^{op}_{LP} \in \Delta^K$ (recall $v^*_w$ is defined as the fixed point of $T^*_w$). Therefore, $p^{op}_{LP} = \sum_{s} \mu_0(s) v^*_{w^{op}_{LP}}(s) = \mathcal{L}(w^{op}_{LP}) \geq \mathcal{L}(w^*)$ by the definition of $w^*$.
        \item There exists the unique $v^{*}_{w^*}$ satisfying $T^*_{w^*}v^{*}_{w^*} = v^{*}_{w^*}$ for the mapping $T^*_{w^*}$ in \eqref{eq:w_bellman_optimality_operator} since $T^*_{w^*}$ is a contraction. Since $w^* \in \Delta^K$ and $(w^*, v^{*}_{w^*})$ satisfies  \eqref{eq:target_LP_2} due to the equivalence between (\ref{eq:target_LP_1},    \ref{eq:target_LP_2}) and  \eqref{eq:equivalence_observation}, $(w^*, v^{*}_{w^*})$ is feasible in the LP \textbf{P1} and  $\mathcal{L}(w^*) = \sum_s \mu_0(s) v^{*}_{w^*}(s) \geq p^{op}_{LP} = \mathcal{L}(w^{op}_{LP})$.
    \end{enumerate}
    Therefore, $\mathcal{L}(w^*) = p^{op}_{LP}$; $(w^*, v^{*}_{w^*})$ is an optimal solution of \textbf{P1}; and $w^{op}_{LP}$ is an optimal solution of \textbf{P2}.
\end{proof}

Another property of  $\mathcal{L}(w)$ is the following piecewise-linearity. 

\begin{theorem}\label{thm_3}
    For each $s$, $v^{*}_{w}(s)$ is a piecewise-linear function in $w \in \mathbb{R}^K$. Consequently, the objective $\mathcal{L}(w) = \sum_s \mu_0(s) v^{*}_{w}(s)$ is also  piecewise-linear in $w \in \mathbb{R}^K$. 
\end{theorem}

\begin{proof}
    Appendix \ref{append:piecewise_linear}.
\end{proof}

\section{Regularization for Max-Min Policy}\label{sec:regularize_maxent}

Suppose we have acquired the optimal $w^*$ from \textbf{P2} and the corresponding optimal action value function  $Q^*_{w^*}$ with optimal scalarization weight $w^*$ such that $\forall s, v^*_{w^*}(s) = \max_a Q^*_{w^*}(s,a)$. Then, $\arg \max_a Q^*_{w^*}(s,a)$ is a deterministic policy which attains the max-min value as follows: 
\begin{equation}
    \mathcal{L}(w^*) = \mathbb{E}_{s \sim \mu_0} [\max_a Q^*_{w^*}(s,a)] = \max_{\pi \in \Pi}  \min_{1 \leq k \leq K} J_k(\pi).
\end{equation}
As shown in Section \ref{subsec:example_reg_need}, however,  $\arg \max_a Q^*_{w^*}(s,a), \forall s$ is not necessarily an optimal max-min policy. To resolve this issue,  we propose a regularized version of the max-min formulation, denoted as \textbf{P0'}, to obtain the optimal max-min policy in Section \ref{subsec:reg_maxent_derive}.

\subsection{An Example of Indeterminacy}\label{subsec:example_reg_need}

\begin{figure}[!ht]
    \begin{center}
        \includegraphics[width=0.28\linewidth]{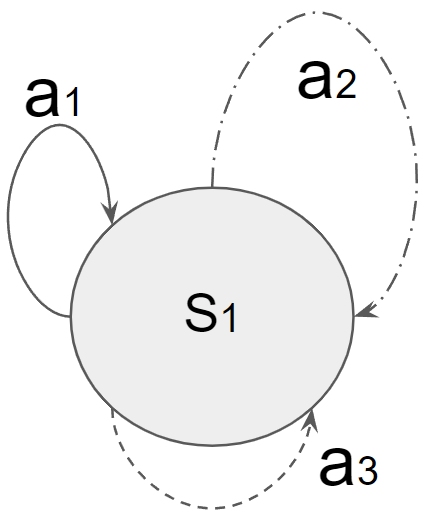} 
        \includegraphics[width=0.71\linewidth]{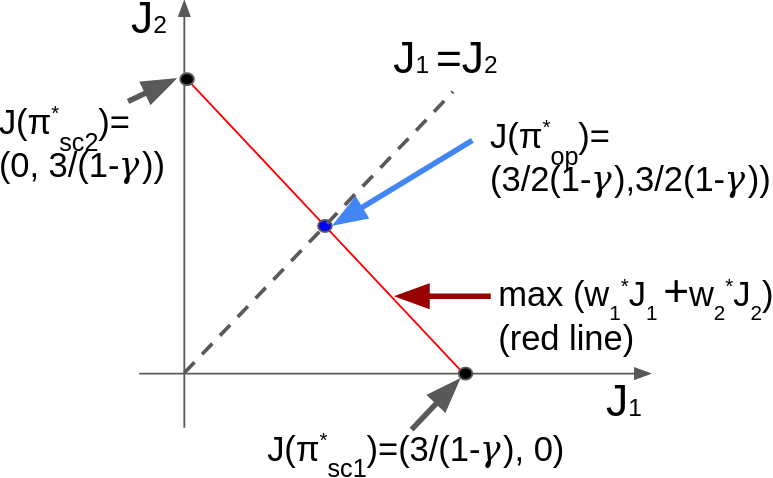}
    \end{center}
    \caption{(Left) one-state example \cite{roijers13survey} and  (Right) cumulative return vectors of $J(\pi_{sc1}^*), J(\pi_{sc2}^*)$, and $\pi^*_{op}$.  }
    \label{fig:one_state}
\end{figure}

Consider the one-state two-objective MDP \cite{roijers13survey} in Fig. \ref{fig:one_state} (Left). Let the initial distribution be $\mu_0(s_1) = 1$ and $0 < \gamma < 1$. We  have reward function $r(s_1, a_{1})=[3,0], r(s_1, a_{2})=[0,3]$, and $r(s_1, a_{3})=[1,1]$.

If we first solve the \textbf{P1} analytically, the exact solution is $v^*(s_1) = \frac{3}{2(1-\gamma)},  w^*_1 = w^*_2 = \frac{1}{2}$. The following policy $\pi^*$ is the optimal policy of MDP whose scalar reward is given by $w^*_1r^{(1)}+w^*_2r^{(2)}$ with $w^* = (w^*_1, w^*_2)=(1/2,1/2)$:
\begin{align} \label{eq:naive_deterministic}
    \pi^*(s_1) &= \mathop{\arg \max}_{a_{i}, 1 \leq i \leq 3} Q^*_{w^*}(s, a_i)  \nonumber \\
    &= \mathop{\arg \max}_{a_{i}} [\frac{3}{2} + \gamma v^*(s_1), \frac{3}{2} + \gamma v^*(s_1), 1 + \gamma v^*(s_1) ] \nonumber \\
    &= a_1 ~ \text{or} ~ a_2.
\end{align}
Let $\pi_{sc1}^*(s_1) = a_1$ and $\pi_{sc2}^*(s_1) = a_2$, both of which are deterministic policies. Then, the corresponding cumulative return vectors are $J(\pi_{sc1}^*) = \left( \frac{3}{1-\gamma}, 0 \right)$ and $J(\pi_{sc2}^*) = \left( 0, \frac{3}{1-\gamma} \right)$, respectively, both of which have the max-min value 0.

On the other hand, the exact solution of the original max-min problem \textbf{P0} is $ d^*(s, a_1) = d^*(s, a_2) = \frac{1}{2(1-\gamma)}, d^*(s, a_3) = 0$. The optimal induced (stochastic) policy is
\begin{equation}
    \pi^*_{op}(a|s_1) =  0.5 ~ \text{for} ~ a = a_1, ~~   0.5 ~ \text{for} ~ a = a_2, ~~ 0  ~ \text{o.w.}
\end{equation}
with the cumulative return $J(\pi^*_{op}) = \left( \frac{3}{2(1-\gamma)}, \frac{3}{2(1-\gamma)} \right)$. 

This example shows that naively solving \textbf{P1} (or equivalent \textbf{P2}) gives the optimal max-min value $\frac{3}{2(1-\gamma)}$ due to the strong duality between \textbf{P0} and \textbf{P1}, but does not necessarily recover the optimal max-min policy of \textbf{P0}. This happens because the primal solution $d^*$ of \textbf{P0} is not explicitly expressed in the solution $(w^*, v^*)$ of the dual problem \textbf{P1} in an 1-to-1 manner in general. Note that any points including $J(\pi^*_{sc1})$, $J(\pi^*_{sc2})$ and $J(\pi_{op}^*)$ on the red line with slope -1  in Fig.      \ref{fig:one_state} (Right) yields the same value for $w^*_1 J_1 + w^*_2 J_2$ with $w^*_1= w^*_2 =1/2$. The max-min point should be simultaneously on this red line with slope -1 and the line $J_1=J_2$.

\subsection{Entropy-Regularized Max-Min Formulation} \label{subsec:reg_maxent_derive}

The indeterminancy of the solution of \textbf{P0} from the solution of \textbf{P1} or \textbf{P2} results from the fact that $d(s,a)$ is not explicitly recovered from the solution of \textbf{P1} or \textbf{P2}. To circumvent this limitation and impose an explicit correspondence of $d^*$ to $(w^*, v^*)$, we add a proper regularization term in \textbf{P0}.  

We choose entropy regularization because (i) the entropy regularization term reformulates \textbf{P0} as a convex optimization, denoted as \textbf{P0'}, (ii) it additionally injects exploration to improve online training \cite{haarnoja2018soft}, and (iii) it is favored over general KL-divergence-based regularization due to its algorithmic simplicity (please see Appendix \ref{append:entropy_regul_simple} for details).

Thus, the new entropy injected problem \textbf{P0'} is given by 
\begin{equation} \label{eq:reg_maxmin_obj}
    \textbf{P0'}: \max_{d} \min_{1 \leq k \leq K}  \sum_{s,a} d(s,a) \left\{ r^{(k)}(s,a) + \alpha \mathcal{H}( \pi^d(\cdot | s)) \right\}  
\end{equation}
\begin{equation} \label{eq:reg_maxmin_transition_eq}
    \sum_{a'} d(s',a') = \mu_0(s') + \gamma \sum_{s,a} P(s' | s,a) d(s,a), ~~ \forall s'
\end{equation}
\begin{equation} \label{eq:reg_maxmin_transition_pos}
    d(s,a) \geq 0, ~~ \forall (s,a) .
\end{equation}
where $\pi^d(a|s) = \frac{d(s,a)}{\sum_{a'}d(s,a')}$ is the policy induced by $d$ \cite{Puterman_2005} and $\mathcal{H}( \pi^d(\cdot | s))$ is its entropy given $s$. Since the objective can be rewritten as $\max_{d} \bigg[ \left\{ \min_{1 \leq k \leq K}  \sum_{s,a} r^{(k)}(s,a) d(s,a) \right\} + \alpha \sum_{s,a}d(s,a) \mathcal{H}(\pi^d(\cdot|s)) \bigg]$ and $\sum_{s,a}d(s,a) \mathcal{H}(\pi^d(\cdot|s))$ is concave regarding $d$ due to the log sum inequality, \textbf{P0'} is a convex optimization. After inserting a slack variable $c=\min_{1 \leq k \leq K}  \sum_{s,a} r^{(k)}(s,a) d(s,a)$, the convex dual problem of \textbf{P0'} is written as follows:
{\small\begin{align} \label{eq:dual_of_maxent_maxmin}
     & \min_{w\geq 0, v} \hspace{-0.1em}\min_{\xi \geq 0} \hspace{-0.1em}\max_{d,c} \hspace{-0.2em}\Bigg[ c(1 - \hspace{-0.1em}\sum_{k=1}^K w_k) \hspace{-0.1em}-\hspace{-0.1em} \alpha \hspace{-0.1em}\sum_{s,a}\hspace{-0.1em} d(s,a) \hspace{-0.1em}\log \frac{d(s,a)}{\sum_{a'} d(s,a')}    \nonumber \\
      & + \hspace{-0.1em}\sum_{s}\hspace{-0.1em}\mu_0(s) v(s) + \sum_{s,a} \xi(s,a)d(s,a) \nonumber \\
      & + \hspace{-0.1em}\sum_{s,a} d(s,a)\hspace{-0.1em} [\sum_{k=1}^K w_k r^{(k)}\hspace{-0.1em}(s,a) \hspace{-0.1em}+\hspace{-0.1em} \gamma \hspace{-0.1em}\sum_{s'} P(s' | s,a) v(s') \hspace{-0.1em}-\hspace{-0.1em} v(s)  ] \Bigg].
\end{align}}
If we apply the stationarity condition to the Lagrangian $L$ which is the whole term in the large brackets in \eqref{eq:dual_of_maxent_maxmin}, we have $\frac{\partial L}{\partial c} = 1 - \sum_{k=1}^K w_k = 0$ and $\forall (s,a) $,
\begin{equation} \label{eq:kkt_stationarity}
    \frac{\partial L}{\partial d(s,a)} = -  \alpha \log \frac{d(s,a)}{\sum_{a'} d(s,a')}  + \xi(s,a) + \eta_{v,w}(s,a) = 0
\end{equation}
where $\eta_{v,\hspace{-0.1em}w}(s,a) \hspace{-0.2em}= \sum_{k} \hspace{-0.1em}w_k r^{(k)}\hspace{-0.1em}(s,a) \hspace{-0.2em}+\hspace{-0.3em} \gamma  \hspace{-0.2em}\sum_{s'} \hspace{-0.2em}P(s' | s,a) v(s')$ $- v(s)$.  \eqref{eq:kkt_stationarity} imposes the explicit connection between $d$ and $(w, v)$. Note that as $\alpha \rightarrow 0$, the connection between $d$ and $(w, v)$ vanishes in \eqref{eq:kkt_stationarity}, and  the convex dual problem \eqref{eq:dual_of_maxent_maxmin} of \textbf{P0'} reduces to  the dual problem  of \textbf{P0-LP}. 

Since $\frac{d(s,a)}{\sum_{a'} d(s,a')} = \exp \left( \frac{\xi(s,a) + \eta_{v,w}(s,a)}{\alpha} \right) > 0$ from (\ref{eq:kkt_stationarity}), we have $\xi(s,a)=0$ due to the complementary slackness condition $d(s,a)\xi(s,a)=0$. After plugging  $\frac{d(s,a)}{\sum_{a'} d(s,a')} = \exp \left( \frac{\eta_{v,w}(s,a)}{\alpha} \right)$ into \eqref{eq:dual_of_maxent_maxmin} and some manipulation, the problem \eqref{eq:dual_of_maxent_maxmin} reduces to the following problem:
\begin{equation} \label{eq:joint_dual_maxent_maxmin_obj}
    \textbf{P1'}: ~~ \min_{w \in \Delta^K, v} \sum_s \mu_0(s) v(s)
\end{equation}
{\small 
\begin{equation} \label{eq:joint_dual_maxent_maxmin}
     v(s) \hspace{-0.2em} = \hspace{-0.2em}\alpha \log \hspace{-0.2em}\sum_a \hspace{-0.2em}\exp [ \frac{1}{\alpha}  \{ \sum_{k=1}^K \hspace{-0.2em}w_k r^{(k)}\hspace{-0.2em}(s,a) + \gamma \hspace{-0.3em}\sum_{s'}\hspace{-0.2em}P(s'|s,a)v(s') \} ]
\end{equation}}where $\Delta^K$ is the $(K-1)$-simplex.  If we solve  \textbf{P1'} and find an optimal solution $(w^*, v^*)$, due to the strong duality under Slater condition \cite{boyd2004convex}, we directly recover the induced optimal policy of \textbf{P0'} as 
{\small\begin{equation} \label{eq:soft_optimal_policy}
   \pi^{*}(a|s) \hspace{-0.1em}=\hspace{-0.1em} \pi^{d^*}\hspace{-0.1em}(a|s)\hspace{-0.1em} = \hspace{-0.1em}\frac{d^*(s,a)}{\sum_{a'} d^*(s,a')} \hspace{-0.1em} = \hspace{-0.1em} \exp \hspace{-0.1em}\left( \frac{\eta_{v^*,w^*}(s,a)}{\alpha} \right).
\end{equation}}The only difference between \textbf{P1} and \textbf{P1'} is that \eqref{eq:target_LP_2} is changed to \eqref{eq:joint_dual_maxent_maxmin}, which implies that the standard value iteration is replaced with the soft value iteration, where the soft Bellman operator is also a contraction \cite{haarnoja2017sql}. 

Now consider the previous example in Section \ref{subsec:example_reg_need} again.
Unlike  \textbf{P1}, solving  \textbf{P1'} of the example indeed yields a near-optimal max-min policy: $\pi^*(a_1 | s_1) = \pi^*(a_2 | s_1) = \frac{1}{2 + \exp(-\frac{1}{2 \alpha})}$ with $\pi^*(a_1 | s_1) = \pi^*(a_2 | s_1) = 0.5$ as $\alpha \rightarrow 0^+$ (please see Appendix \ref{append:reg_sol_one_state} for the detailed derivation).

Similarly to Section \ref{subsec:P1_to_P2}, we then consider the following  optimization:
\begin{equation} \label{eq:soft_equivalent_convex}
    \textbf{P2'}: ~ \min_{w \in \Delta^K} \mathcal{L}^{soft}(w) = \sum_s \mu_0(s) v^{soft,*}_{w}(s)
\end{equation}
where $\Delta^K$ is the $(K-1)$-simplex and $v^{soft,*}_{w}$ is  the unique fixed point of the soft Bellman optimality operator $\mathcal{T}^{soft,*}_{w}$ defined as   $(\mathcal{T}^{soft,*}_{w} v)(s) := \alpha \log \sum_a \exp [ 1/\alpha * \{ \sum_{k=1}^K w_k r^{(k)}(s,a) + \gamma \sum_{s'}P(s'|s,a)v(s') \} ],\forall s$ for a given $w$.

\begin{theorem}\label{thm_4}
    For each $s$, $v^{soft,*}_{w}(s)$ is a convex function with respect to $w \in \mathbb{R}^K$. Consequently, the objective $\mathcal{L}^{soft}(w) = \sum_s \mu_0(s) v^{soft,*}_{w}(s)$ is also convex with respect to $w \in \mathbb{R}^K$. 
\end{theorem}

\begin{proof}
    Appendix \ref{append:convexity_soft}.
\end{proof}

\begin{theorem}\label{thm_5}
    Solving \textbf{P2'} is equivalent to solving  \textbf{P1'}.
\end{theorem}

\begin{proof}
    Given $w \in \Delta^K$, the only feasible $v$ satisfying \eqref{eq:joint_dual_maxent_maxmin} is $v = v^{soft,*}_w$. Plugging $v = v^{soft,*}_w$ in \eqref{eq:joint_dual_maxent_maxmin_obj} gives \eqref{eq:soft_equivalent_convex}.
\end{proof}

Unlike in Theorem \ref{thm_3} stating $\mathcal{L}(w)$ is piecewise-linear in the unregularized case, $v^{soft,*}_{w}(s)$ and  hence $\mathcal{L}^{soft}(w)$ with entropy regularization are  continuously differentiable  with respect to $w \in \Delta^K$, as shown in the following theorem.

\begin{theorem}\label{thm_6}
    For each $s$, $v^{soft,*}_w(s)$ is a continuously differentiable function with respect to $w \in \mathbb{R}^K$. Consequently, the objective $\mathcal{L}^{soft}(w) = \sum_s \mu_0(s) v^{soft,*}_w(s)$ is also continuously differentiable function with respect to $w \in \mathbb{R}^K$.
\end{theorem}

\textit{Proof sketch.} 
The theorem follows by applying the implicit function theorem with the fact that $\mathcal{T}^{soft,*}_{w}$ has the unique fixed point for each $w$ (please see Appendix \ref{append:continuous_differentiable} for the details).

Hence, $\mathcal{L}^{soft}(w)$ is not piecewise-linear. However, $v^{*}_{w}(s)$ and thus $\mathcal{L}^{soft}(w)$ have Lipschitz continuity, as shown in Theorem \ref{thm_7}, which is the property of  any piecewise-linear function with finite segments. Lipschitz continuity is a core condition for our proposed method in Section \ref{sec:model_free}.

\begin{theorem}\label{thm_7}
    For each $s$, $v^{soft,*}_w(s)$ is Lipschitz continuous as a function of $w$ on $\mathbb{R}^K$ in $\| \cdot \|_{\infty}$, and so is $\mathcal{L}^{soft}(w)$.
\end{theorem}

\begin{proof}
    Appendix \ref{append:lipschitz_continuity}.
\end{proof}

Suppose we have solved \textbf{P2'} and obtained the optimal $(w^*, v^{soft,*}_{w^*})$. We then explicitly recover the optimal policy of \textbf{P0'} as \eqref{eq:soft_optimal_policy}. 
The soft Q-function \cite{haarnoja2017sql} $Q^{soft,*}_{w^*}$  corresponding to $v^{soft,*}_{w^*}$ satisfies the soft Bellman equation $Q^{soft,*}_{w^*}(s,a) = \sum_{k=1}^K w^{*}_k r^{(k)}(s,a) + \gamma \sum_{s'} P(s' | s,a) v^{soft,*}_{w^*}(s')$. Then, by dividing $\alpha$ and taking exponential on the both sides of \eqref{eq:joint_dual_maxent_maxmin}, 
 \eqref{eq:joint_dual_maxent_maxmin} can be  rewritten
as $\sum_{a'} \exp \left( \frac{Q^{soft,*}_{w^*}(s,a') - v^{soft,*}_{w^*}(s)}{\alpha} \right) = 1, ~ \forall s$. Since $\eta_{v^{soft,*}_{w},w}(s,a)= Q^{soft,*}_{w}(s,a)-v^{soft,*}_{w}(s)$ as seen just below \eqref{eq:kkt_stationarity}, the optimal policy $\pi^{*}$ of \textbf{P0'} is written as $\pi^{*}(a|s) = \exp \left( \frac{Q^{soft,*}_{w^*}(s,a) - v^{soft,*}_{w^*}(s)}{\alpha} \right)$, or
\begin{equation}
    \pi^{*}(a|s) = \text{softmax}_a  \{ Q^{soft,*}_{w^*}(s,a) / \alpha  \}.
\end{equation}

Note that \textbf{P2'} is basically weight optimization combined with soft value iteration. Thus, \textbf{P2'} is the  basis from which we derive our model-free max-min MORL algorithm. The overall development procedure is summarized in Fig. \ref{idea_flow}. 

\begin{figure}[ht]
% \vskip 0.2in
\begin{center}
\centerline{\includegraphics[width=\columnwidth]{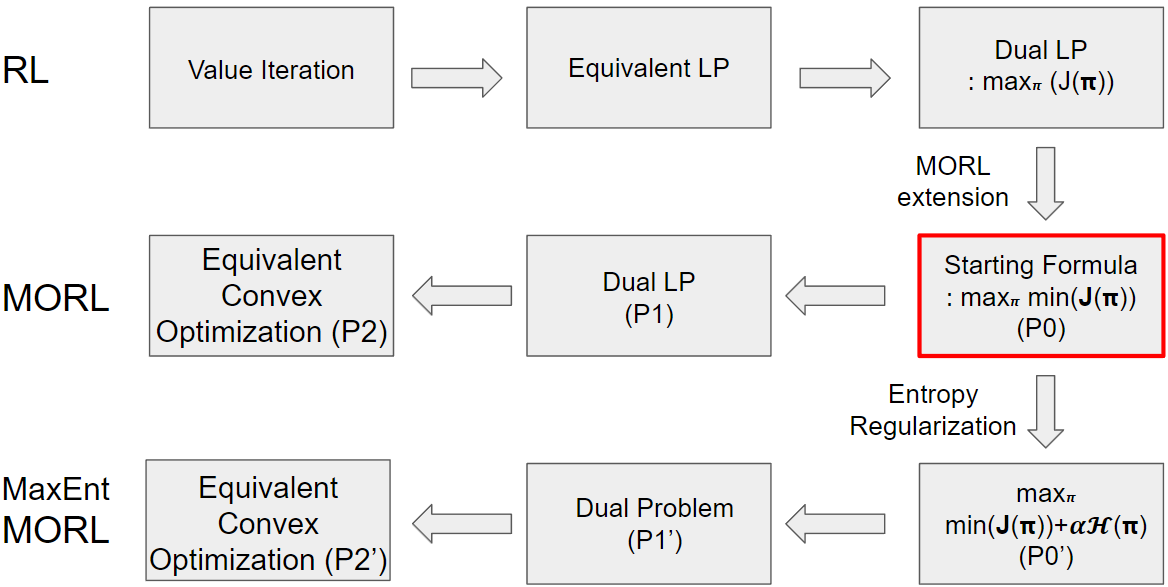}}
\caption{Our formulation procedure of the max-min problem.}
\label{idea_flow}
\end{center}
\vskip -0.2in
\end{figure}

\section{The Proposed Model-Free Algorithm}\label{sec:model_free}

Our key idea to solve \textbf{P2'} and obtain a model-free value-based max-min MORL algorithm  is the alternation between $Q^{soft}_w$ update with scalarized reward for given $w$ and the $w$ update for given $v^{soft}_w = \mathbb{E}_{s \sim \mu_0} [ \alpha \log \sum_a \exp [Q^{soft}_{w}(s,a) / \alpha]]$.  For the $Q^{soft}_w$ update for given $w$, we adopt the soft Q-value iteration \cite{haarnoja2017sql}. Thus, we need to devise  a stable method for the $w$ update for  given $v^{soft}_w$.

\subsection{Gradient Estimation Based on Gaussian Smoothing}

A basic $w$ update  method to solve \textbf{P2'} is gradient descent with  the gradient  $\nabla_w \mathcal{L}^{soft}(w) \bigg\rvert_{w = w^m}$ at the $m$-th step, and updates  $w^m$ to $w^{m+1}$ by using the gradient, where $\mathcal{L}^{soft}(w) = \sum_s \mu_0(s) v^{soft,*}_{w}(s) = \mathbb{E}_{s \sim \mu_0} [ \alpha \log \sum_a \exp [Q^{soft,*}_{w}(s,a) / \alpha]]$ is the objective function acquired from soft Q-learning \cite{haarnoja2017sql}. Here, $Q^{soft,*}_{w}$ is the unique fixed point of the soft Bellman operator with the scalarization weight $w$ satisfying the soft Bellman equation $Q^{soft,*}_{w}(s,a) = \sum_{k=1}^K w_k r^{(k)}(s,a) + \gamma \sum_{s'} P(s' | s,a) v^{soft,*}_{w}(s')$ (the convergence of soft Q-value iteration is guaranteed by \citet{fox2016taming,haarnoja2017sql}). However,  the closed form of $Q^{soft,*}_{w}$ (and  consequently $\mathcal{L}^{soft}(w)$) with respect to $w$ is unknown. Hence,  deriving $\nabla_w \mathcal{L}^{soft}(w)$ is  challenging. 

To circumvent this difficulty, numerical computation of gradient can be employed.   A naive approach is the dimension-wise finite difference gradient estimation \cite{silver2015} in which the gradient is estimated as $\frac{\partial \mathcal{L}^{soft}(w)}{\partial w_k} \approx \frac{1}{\epsilon}(\mathcal{L}^{soft}(w+\epsilon e_k) - \mathcal{L}^{soft}(w)),\forall k$, where  $e_k$ is the one-hot vector with $k$-th dimension value 1. However, this method is sensitive to function noise and has a tendency to produce unstable estimation \cite{silver2015}.

In order to have a stable gradient estimation, we propose a novel gradient estimation based on linear regression. Given a current weight point $w^m \in \mathbb{R}^K$ at $m$-th step, we generate $N$ perturbed samples $\{ w^m + \mu u_i^m  \}_{i=1}^N$, where $u_i^m \sim \mathcal{N}(0,I_K)$ with the identity matrix $I_K$ of size $K$, and $\mu > 0$ is a perturbation size parameter. Using the input samples $\{ w^m + \mu u_i^m  \}_{i=1}^N$,  we compute the output values $\{ \mathcal{L}^{soft}( w^m + \mu u_i^m  )  \}_{i=1}^N$ of the function $\mathcal{L}^{soft}(w)$ and  obtain a linear regression function $h_m(w) = a_m^T w + b_m$ from the input to output values. Then, we use the linear coefficient $a_m$ as an estimation of  $\nabla_w \mathcal{L}^{soft}(w) \bigg\rvert_{w = w^m}$ and update the weight as $w^{m+1} =  \text{proj}_{ \Delta^K } \left(w^{m} - l_m  a_m \right)$, where $l_m$ is a learning rate at the $m$-th step and $\text{proj}_{ \Delta^K}$ is the projection onto the $(K-1)$-simplex. 

The validity of the proposed gradient estimation method is provided by  the concept of Gaussian smoothing \cite{nesterov2017random}. For  a convex (possibly non-smooth) function $g:\mathbb{R}^K\to\mathbb{R}$,  its Gaussian smoothing  is defined as $g_\mu(x):=\mathbb{E}_{u\sim\mathcal{N}(0,I_K)}[g(x+\mu u)]$, where $\mu > 0$ is a smoothing parameter. Then, 
$g_\mu$ is  convex due to the convexity of $g$, and is  an upper bound of $g$. If $g$ is $L_0$-Lipschitz continuous, then the gap between $g$ and $g_\mu$ is $\|g_\mu(x) - g(x)\| \leq \mu L_0 \sqrt{K}, ~ x \in \mathbb{R}^K$ \cite{nesterov2017random}. Furthermore,   $g_\mu$ is differentiable and its  gradient is given by   $\nabla g_\mu(x)=\mathbb{E}_{u\sim\mathcal{N}(0,I_K)}[\frac{1}{\mu}g(x+\mu u) u]$ \cite{nesterov2017random}. Note that  we do not know the exact form of $g$ but we only know the value $g(x)$ given input $x$. 

Theorem \ref{thm_8} provides an interpretation of our gradient estimation in terms of Gaussian smoothing: 

\begin{theorem}\label{thm_8}
    Given a current point $x \in \mathbb{R}^K$,  let $a$ be the  coefficient vector  of linear regression
$h(x) = a^T x + b$ for $N$ perturbed input points $\{ x + \mu u_i  \}_{i=1}^N$, and $N$ corresponding output values $\{ g(x + \mu u_i)  \}_{i=1}^N$ of a function $g$,  where $u_i \sim \mathcal{N}(0,I_K)$ and $\mu > 0$. Then,  
$a \rightarrow \nabla g_\mu(x)$ as $N \rightarrow \infty$, where $g_\mu$ is the Gaussian smoothing of $g$. 
\end{theorem}

\begin{proof}
    Since $u_i \sim i.i.d.~\mathcal{N}(0,I_K)$, by law of large numbers, we have $\frac{1}{N}\sum_i u_i\to 0, \frac{1}{N}\sum_i u_i(u_i)^T = \frac{1}{N}\sum_i (u_i - 0)(u_i - 0)^T \to \mathbb{E}_{u \sim \mathcal{N}(0,I_K)} [(u - 0)(u - 0)^T] = I_K$. If we solve the linear regression $\min_{a,b}\sum_i\{a^T(x + \mu u_i)+b - g(x + \mu u_i)\}^2$, we have $a
    =\frac{1}{\mu}(\frac{1}{N}\sum_i u_i(u_i)^T-\frac{1}{N^2}\sum_i u_i \sum_i (u_i)^T)^{-1}(\frac{1}{N}\sum_i g(x+\mu u_i)u_i-\frac{1}{N^2}\sum_i g(x+\mu u_i) \sum_i u_i)
    \to \frac{1}{\mu}(I_K-0\cdot 0^T)^{-1}(\mathbb{E}_{u\sim\mathcal{N}(0,I_K)}[ g(x+\mu u) u]-g_\mu(x)\cdot0)
    = \nabla g_\mu(x)$.
\end{proof}

Since our gradient estimate approximates $\nabla g_\mu(w)$ with $g(w)= \mathcal{L}^{soft}(w)$, the proposed method ultimately finds the optimal value of the Gaussian smoothing of $\mathcal{L}^{soft}(w)$ which approximates the optimal value of $\mathcal{L}^{soft}(w)$ with the gap is $O(\mu)$ under the assumption of the Lipschitz continuity of $\mathcal{L}^{soft}(w)$. Note that the Lipschitz continuity of $\mathcal{L}^{soft}(w)$ is guaranteed by  
Theorem \ref{thm_7}. 

The proposed algorithm based on alternation between $w$ update and soft Q-value update is summarized in Algorithm \ref{alg:maxent_maxmin}. Our source code is provided at \url{https://github.com/Giseung-Park/Maxmin-MORL}. For projected gradient decent onto the simplex, we use the optimization technique from \citet{DBLP:journals/corr/WangC13a}. 

Note that the $N$ perturbed weights in Line 6 of Algorithm \ref{alg:maxent_maxmin} do not deviate much from the current weight $w^m$. So, in our implementation, we perform one step of gradient update for Soft Q-learning in Line 8 for each copy $\hat{Q}_{w^m, \text{copy}, n}$. Thus, the overall complexity of the proposed algorithm is at the level of Soft Q-learning with slight increase due to linear regression at each step (please see Appendix \ref{subappend:additional_ablation} for the analysis on computation).

%%%%%%%%%%%%%%%
\begin{algorithm}[t] 
\caption{Proposed Max-min Model-free Algorithm} \label{alg:maxent_maxmin}
\begin{algorithmic}[1]
    \STATE $K$: reward dimension, $\hat{Q}$: initialized soft Q-network, $\mathcal{M}$: replay buffer,  $U$: iteration number, $N$: number of perturbed samples, $\mu$: perturbation parameter, $l_0$: initial learning rate of the weight $w$.
    \STATE Initialize weight $w^0 \in \Delta^K$ (e.g., uniform).
    \STATE Update $\hat{Q}$ using soft Q-learning with weight $w^{0}$ and acquire $\hat{Q}_{w^{0}, \text{main}}$. Save rollout samples in $\mathcal{M}$.
    \FOR{$m = 0, 1, 2, \cdots, U-1$}
        \STATE{Rollout sample from $\hat{Q}_{w^m, \text{main}}$ and save it in $\mathcal{M}$.}
        \STATE Generate $N$ perturbed weights $\{ w^m + \mu u^{m}_n \}_{n=1}^{N}$, $u^{m}_n \sim \mathcal{N}(0, I_K)$.
        \STATE Make $N$ copies of $\hat{Q}_{w^{m}, \text{main}}: \{ \hat{Q}_{w^m, \text{copy}, n} \}_{n=1}^{N}$.
        % \STATE Make a copy of the main target function.
        \STATE Update each $\hat{Q}_{w^m, \text{copy}, n}$ by soft Q-learning with $w^m + \mu u^{m}_n$ using common samples in $\mathcal{M}$ and target function to acquire $\hat{Q}_{w^m + \mu u^{m}_n, \text{copy}, n}$. 
        \STATE Calculate $\hat{L}(w^m + \mu u^{m}_n) = \mathbb{E}_{s \sim \mu_0}  [ \alpha \log \sum_a \exp [\hat{Q}_{w^m + \mu u^{m}_n, \text{copy}, n}(s,a) / \alpha]]$.
        \STATE Conduct linear regression using $\{w^m + \mu u^{m}_n, \hat{L}(w^m + \mu u^{m}_n)  \}_{n=1}^{N}$ and calculate the linear weight $a_m$.
        \STATE Discard $\{ \hat{Q}_{w^m + \mu u^{m}_n, \text{copy}, n} \}_{n=1}^{N}$.
        \STATE Update $w^m$ using the projected gradient descent:
        \begin{equation*} \label{eq:proj_grad}
            w^{m+1} =   \text{proj}_{ \Delta^K } \left(w^{m} - l_m  a_m  \right).
        \end{equation*}
        \STATE Schedule current learning rate of the weight $l_m$.
        \STATE Update $\hat{Q}_{w^m, \text{main}}$ using soft Q-learning with weight $w^{m+1}$ and acquire $\hat{Q}_{w^{m+1}, \text{main}}$.
    \ENDFOR
    \STATE Return $\hat{\pi}^*(a|s) = \text{softmax}_a  \{ \hat{Q}_{w^{U}, \text{main}}(s,a) / \alpha  \}$.
\end{algorithmic}
\end{algorithm}

\section{Experiments}

\subsection{Max-Min Performance}\label{subsec:performance}

For comparison with our value-based method, we consider the following value-based baselines: (i) Utilitarian, which is a standard Deep Q-learning (DQN) \cite{mnih13} using averaged rewards $\frac{1}{K} \sum_{k=1}^K r^{(k)}$, and (ii) Fair Min-DQN (MDQN), an extension of the Fair-DQN concept \cite{siddique20fair} to the max-min fair metric maximizing $\mathbb{E}[\min_{1 \leq k \leq K} \sum_t \gamma^t r^{(k)}_t]$. For performance evaluation, we calculate the empirical value of $\min_{1 \leq k \leq K} \mathbb{E} [\sum_t \gamma^t r^{(k)}_t]$, where $\mathbb{E} [\sum_t \gamma^t r_t]$ is calculated with five random seeds.

\begin{figure}[ht]
    \begin{center}
        \includegraphics[width=0.6\linewidth]{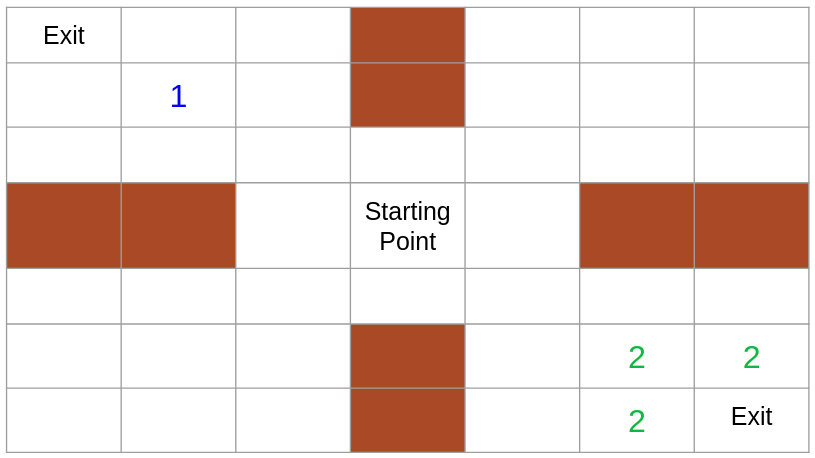}
        \includegraphics[width=0.7\linewidth]{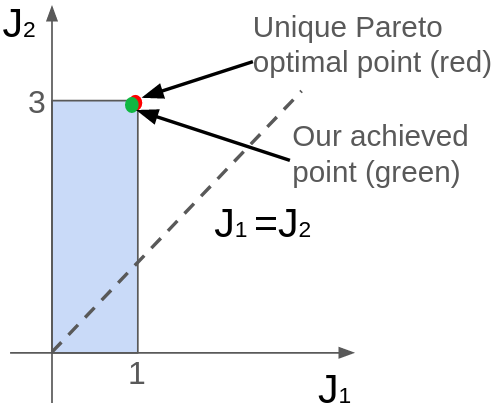}
    \end{center}
    \caption{ (Up) Four-Room environment \cite{felten_toolkit_2023} and (Down) achievable return region in the Four-Room environment (light blue), the unique Pareto optimal point (red dot), and the point our algorithm achieved: $(J_1, J_2) = (0.96, 2.88)$ (green dot).  }
    \label{fig:four_room_main_body}
\end{figure}

First, we consider Four-Room \cite{felten_toolkit_2023}, a widely used MORL environment. The goal is to collect as many elements as possible in a square four-room maze within a given time. As seen in Fig. \ref{fig:four_room_main_body} (Up), there exist two types of elements: Type 1 and Type 2. In our experiment, we have four elements in total: one element of Type 1 and three elements of Type 2. The reward vector has dimension 2, and the  $i$-th dimension of the reward is +1 if an element of Type $i$ is collected, and 0 otherwise. We strategically clustered the three elements of Type 2 near one exit, while the one element of Type 1 was positioned near the other exit. We intend to challenge the agent to balance its collection strategy to prevent it from overly favoring Type 2 over Type 1. The metric is calculated over the 200 most recent episodes.

\begin{table}[t]
\centering
\begin{tabular}{|c|c|c|c|}
\hline
 & Type 1 (max 1) & Type 2 (max 3) & Min \\
\hline
Proposed & 0.96 & 2.88 & \textbf{0.96 } \\
MDQN & 0.64 & 0.60 & 0.60 \\
Utilitarian & 0.76 & 2.56 & 0.76 \\
\hline
\end{tabular}
\caption{Performance in Four-Room environment.}
\label{tab:Experiment_four_room_common_main_body}
\vskip -0.2in
\end{table}

As shown in Table \ref{tab:Experiment_four_room_common_main_body}, the return vector of our method Pareto-dominates the  two return vectors of the other two conventional algorithms. Note that the performance of conventional MDQN is  poor. This is because  MDQN performs  $\max_{a'} \min_{1 \leq k \leq K}[r^{(k)} + \gamma Q^{(k)}(s',a')]$. Suppose that the Q-network is initialized as all zero values. Then,  $\min_{1 \leq k \leq K}[r^{(k)} + \gamma Q^{(k)}(s',a')]$ becomes non-zero only when both types of elements are collected and Q-function is updated only when this happens. However, this event is rare in the initial stage and the learning is slow.  This example shows the limitation of conventional MDQN for max-min MORL. Note that our algorithm almost achieves the unique Pareto-optimal point of this problem, as shown in Fig. \ref{fig:four_room_main_body} (Down).

\begin{figure}[!ht]
  \centering
    \subfigure[]{
            \includegraphics[width=2.7cm,height=2.7cm]{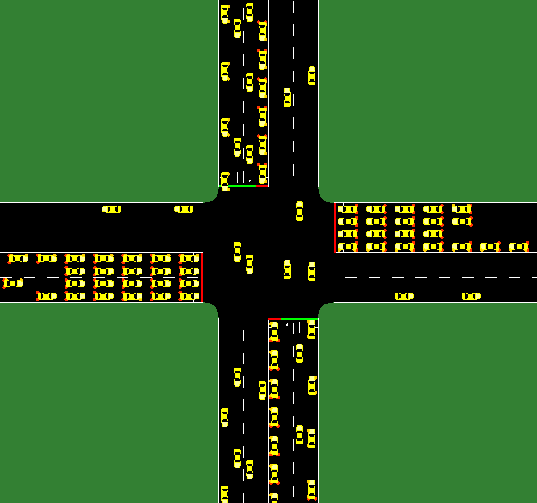}
            \label{fig:traffic_overview}
    }    
    \subfigure[]{
    \includegraphics[width=4.9cm,height=2.7cm]{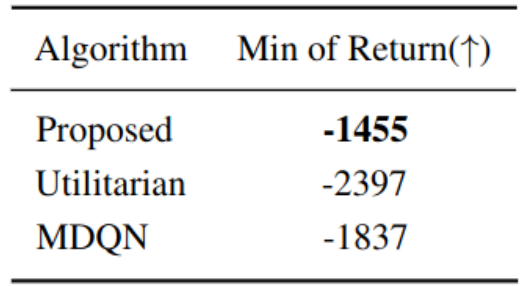}
    \label{tab:performance_traffic_min}
    }
    \subfigure[]{
\centerline{\includegraphics[width=0.9\columnwidth]{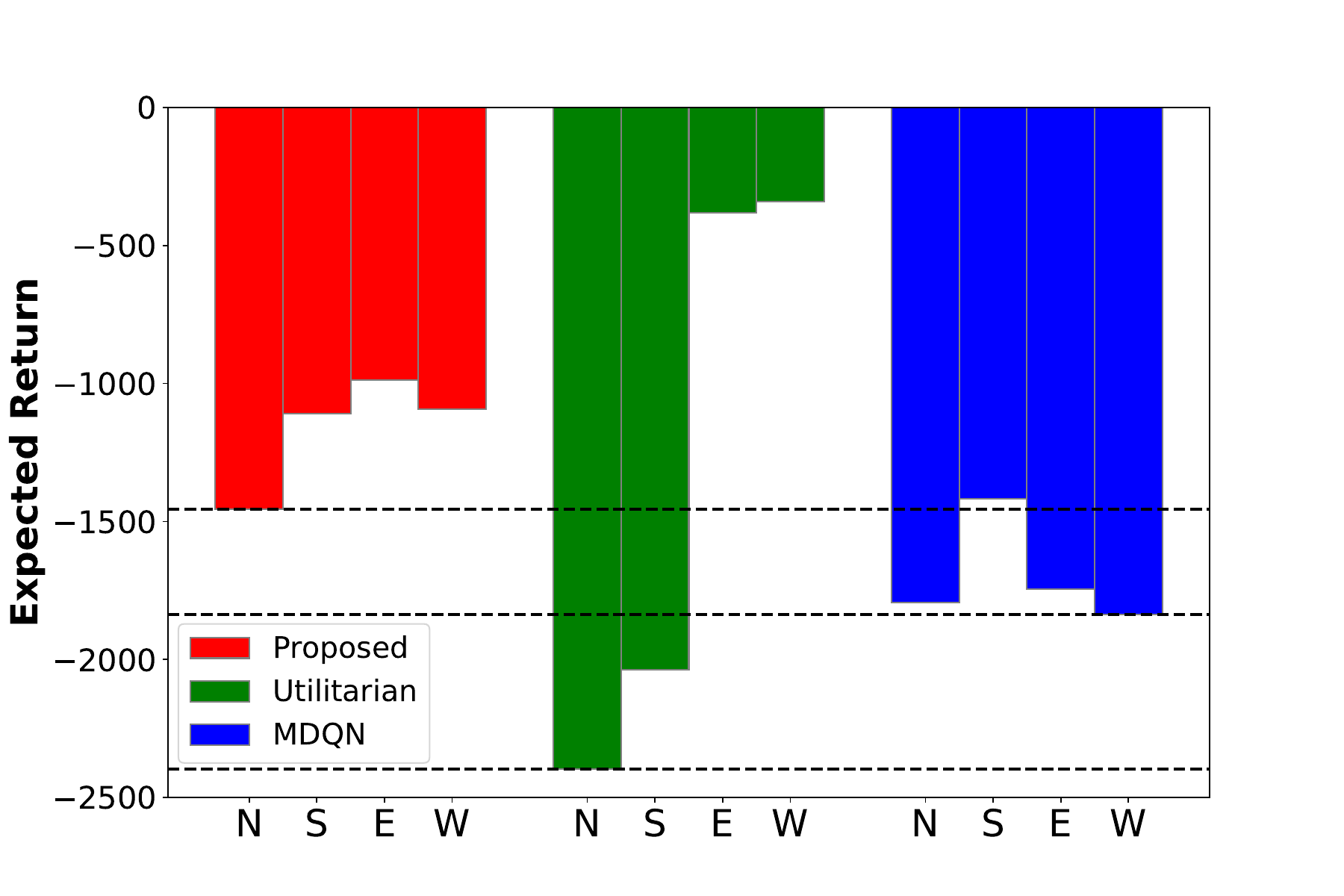}}
    \label{fig:performance_traffic_each_direction}
    }
    \subfigure[]{
     \begin{tabular}{c|cccc|c}
        \hline
        Direction & N & S & E & W & Sum\\
        \hline
        Learned weight & 0.43 & 0.30 & 0.15 & 0.12 & 1\\
        \hline
    \end{tabular}
    \label{tab:learned_weights}
    }
  \caption{ (a) Traffic light control task under consideration, (b) Minimum value of the expected discounted return vector across four dimensions, (c) Expected discounted return for each direction, and (d) Average value of the learned weights of the proposed algorithm. In (c), each black dashed line for each algorithm represents the minimum value of the return across four dimensions.  }
  \label{fig:performance_traffic}
  \vskip -0.2in
\end{figure}

Next, as a realistic multi-objective environment, we consider the traffic light control simulation environment, illustrated
in Fig. \ref{fig:traffic_overview}  \cite{sumorl}. The intersection comprises four road directions (North, South, East, West), with each road containing four inbound and four outbound lanes. At each time step, the agent receives a state containing information about traffic flows. The traffic light controller then selects its traffic light phase as its action. 

The reward is a four-dimensional vector, with each dimension representing a quantity proportional to the negative of the total waiting time for cars on each road. The objective of the traffic light controller is to adjust the traffic signals to minimize the cumulative discounted sum of rewards. We configured the traffic flow to be asymmetric, with a higher influx of cars from the North and South compared to those from the East and West. The metric is calculated over the 32 most recent episodes. (Please see Appendix \ref{subappend:environment} for details on the considered traffic light control environment and Appendix \ref{subappend:implement} for the implementation details.)

Table \ref{tab:performance_traffic_min} shows that the proposed method achieves better max-min performance than the other baselines. Fig. \ref{fig:performance_traffic_each_direction} shows the expected return per direction for each algorithm. Unlike  the proposed method, Utilitarian exhibits a larger gap between the North-South and East-West return values. As shown in Table \ref{tab:learned_weights}, the proposed method assigns larger weight values to North and South. On the other hand, the Utilitarian approach utilizes averaged rewards over dimensions (i.e., weight 0.25  for each direction), resulting in a relatively smaller weight on North-South and a larger weight on East-West, thereby widening the gap between North-South and East-West return values. The standard deviation over the four dimensions is 174.8 for the proposed method and 937.2 for Utilitarian. Compared with Utilitarian, MDQN demonstrates better performance in North-South but worse performance in East-West. Furthermore, the performance of non-minimum other lanes of MDQN is far worse than that of the proposed method. 
Overall, the proposed method achieves the best minimum performance  and shifts up the non-minimum dimension performance  by doing so.

For additional experiments in Species Conservation \cite{siddique23fair}, another widely used MORL environment, please see Appendix \ref{subappend:additional_experiments}. 

\subsection{Ablation Study}\label{subsec:ablation_study}

\begin{figure}[ht]
% \vskip 0.2in
\begin{center}
\centerline{\includegraphics[width=0.9\columnwidth]{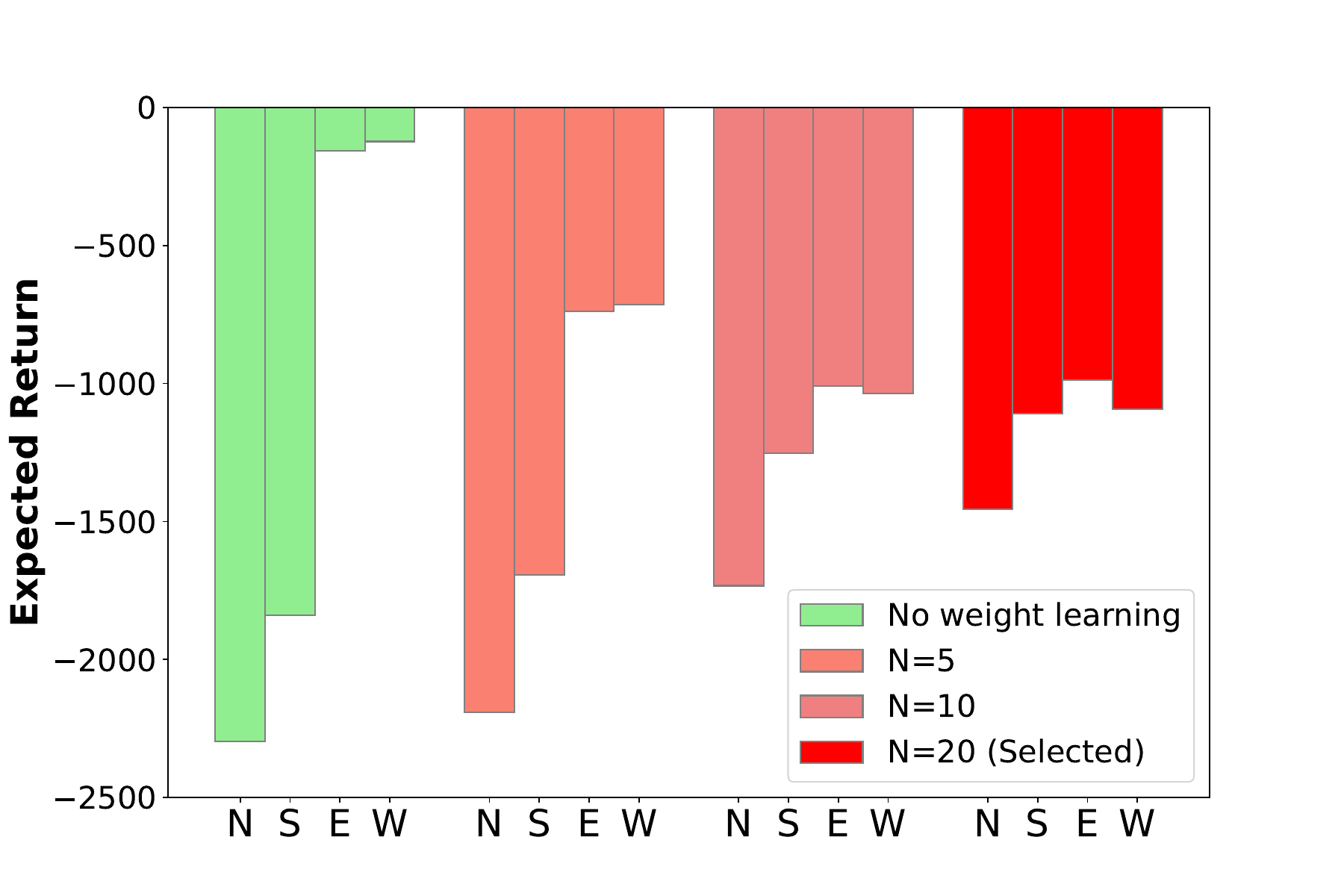}}
\caption{Ablation study on the effect of weight learning and the number of perturbed samples $N$ (the same traffic light control task as in Section \ref{subsec:performance})}
\label{fig:ablation_bar}
\end{center}
\vskip -0.2in
\end{figure}

We examined the impact of weight learning on the performance, which constitutes one of the core components of our proposed approach. We conducted an ablation study by disabling the weight learning update, which resulted in training the algorithm with uniformly initialized weights across directions, while keeping other parts of the algorithm the same. Additionally, we varied the number of perturbed samples $N$ for linear regression,  discussed in Section \ref{sec:model_free}.

{\em Impact of Weight Learning} ~~As shown in Fig. \ref{fig:ablation_bar}, when weight learning was disabled, the gap between the North-South and East-West return values increases. This phenomenon is due to the uniformly initialized weights, leading to performance characteristics similar to those of the Utilitarian approach shown in Fig. \ref{fig:performance_traffic_each_direction}.

{\em Impact of $N$ on $w$ Gradient Estimation} ~~ As the number of perturbed samples $N$ increased, the gap between the North-South and East-West return values decreased, resulting in improved minimum performance. Thus, a sufficient  $N$ (around 20) is required to yield an accurate  $w$ gradient estimate by the proposed approach  outlined in Section \ref{sec:model_free}.

\section{Related Works}

The prevailing trend in MORL is the utility-based approach \cite{roijers13survey}, where the objective is to find an optimal policy $\pi^* = \arg \max_{\pi} f(J(\pi))$ given a non-decreasing scalarization function $f: \mathbb{R}^K \rightarrow \mathbb{R}$. Prioritizing user preferences or welfare aligns well with practical applications \cite{hayes22survey}.

When $f$ is linear, i.e., $f(J(\pi)) = \sum_{k=1}^K w_k J_k(\pi)$ with $\sum_k w_k =1, ~w_k \ge 0,\forall k$, each non-negative weight vector $w$ generates a scalarized MDP where an optimal policy exists \cite{boutilier99optimal}. This formulation simplifies the solution process using standard RL algorithms, shifting the research focus towards acquiring a single network that can produce multiple optimal policies over the weight vector space \cite{abels19moq,yang19envq}. \citet{yang19envq} proposed a multi-objective optimality operator and extended the standard Bellman optimality equation in a multi-objective setting with linear $f$. Subsequent works have addressed two main challenges in this setting: sample efficiency \cite{basaklar23pdmorl,hung23qpensieve} and learning stability \cite{lu23capql}.

When $f$ is non-linear, formulating Bellman optimality equations becomes non-trivial due to the restriction on linearity \cite{roijers13survey}. While some works attempt to develop value-based approaches when $f$ represents certain welfare functions \cite{siddique20fair,fan23welfare}, these methods are related to optimizing $\mathbb{E}_\pi \left[ f( \sum_{t=0}^\infty \gamma^t {r}_t ) \right]$, rather than $f(J(\pi)) = f( \mathbb{E}_\pi \left[ \sum_{t=0}^\infty \gamma^t {r}_t \right] )$, which upper bounds $\mathbb{E}_\pi \left[ f( \sum_{t=0}^\infty \gamma^t {r}_t ) \right]$ when $f$ is concave. In contrast, we propose a value-based method that explicitly optimizes $f(J(\pi))$ when $f$ represents the minimum function.

\section{Conclusion}
We have considered the max-min formulation of MORL to ensure fairness among multiple objectives in MORL. We approached   the problem based on  linear programming and convex optimization  and derived the joint problem of weight optimization and soft value iteration equivalent to the original max-min problem with entropy regularization. We developed a model-free max-min MORL algorithm that alternates weight update with  Gaussian smoothing gradient estimation and soft value update. The proposed method well achieves the max-min optimization goal and yields better performance than baseline methods in the max-min sense.

% \newpage
\section*{Acknowledgements}
This work was supported by the National Research Foundation of Korea (NRF) grant funded by the Korea government (MSIT) (2022K1A3A1A31093462), and 
the Ministry of Innovation, Science \& Technology, Israel and ISF grant 2197/22. 

\section*{Impact Statement}
This paper considers the max-min formulation of MORL and derives a relevant theory and a practical and efficient model-free algorithm for MORL. Since many real-world control problems are formulated as multi-objective optimization, the proposed max-min MORL algorithm can significantly contribute to solving many real-world control problems such as the traffic signal control shown in our experiment section and thus  building an energy-efficient better society.

\bibliography{references}
\bibliographystyle{icml2024}

%%%%%%%%%%%%%%%%%%%%%%%%%%%%%%%%%%%%%%%%%%%%%%%%%%%%%%%%%%%%%%%%%%%%%%%%%%%%%%%
%%%%%%%%%%%%%%%%%%%%%%%%%%%%%%%%%%%%%%%%%%%%%%%%%%%%%%%%%%%%%%%%%%%%%%%%%%%%%%%
% APPENDIX
%%%%%%%%%%%%%%%%%%%%%%%%%%%%%%%%%%%%%%%%%%%%%%%%%%%%%%%%%%%%%%%%%%%%%%%%%%%%%%%
%%%%%%%%%%%%%%%%%%%%%%%%%%%%%%%%%%%%%%%%%%%%%%%%%%%%%%%%%%%%%%%%%%%%%%%%%%%%%%%
\newpage
\appendix
\onecolumn

\newpage
\section{The Wide Use of Max-Min Approach} \label{append:wide_use_maxmin}

\subsection{Practical Applications}
The max-min approach to  multi-objective optimization  has been widely adopted in many practical applications.  Most notably, it has been widely used in \textbf{resource allocation} problems in wireless communication networks (e.g., \citet{zehavi2013weighted}) and scheduling for which RL is being actively investigated as a new control mechanism.   

For example, in scheduling cloud computing resources, a job is typically parsed into multiple tasks which form a directed acyclic graph (DAG, \citet{saifullah14dag}; \citet{wang19dag}), representing the dependencies. In these cases, we need to allocate resources/servers so that dependent tasks will minimize the maximal time of a task among the tasks required to move to the next task in the DAG. This implies that the natural metric is minimizing the delay of the worst user. This problem is most naturally formulated using the max-min formulation, where we aim to maximize the minimal negative delay. In many cases, jobs are repetitive and one needs to optimize the allocation without knowing the statistics of each job on each machine.  

Similarly, when we are providing service to multiple users where we contract all users the same data rate (similarly to Ethernet which has a fixed rate), we would like to maximize the rate of the worst user.  We believe that our max-min MORL algorithm can be used for such resource allocation problems immediately once the problems are set up as RL.

Finally, our max-min MORL approach can provide an alternative way to cooperative \textbf{multi-agent RL} (MARL) problems with central training with distributed execution (CTDE). Currently, in most cooperative MARL, it is assumed that all agents receive the commonly shared scalar reward,  and this causes the lazy agent problem because even if some agents are doing nothing, they still receive the commonly shared reward.  Under CTDE, we can approach cooperative MARL by letting each agent have its individual reward and collecting individual rewards as the elements of a vector reward. Then, we can apply our max-min MORL approach. This is a promising research direction to solve the lazy-agent problem in cooperative MARL.

\subsection{Restoring the Pareto-Front from the Max-Min Approach}

The max-min solution typically yields the equalizer rule \cite{zehavi2013weighted}. That is, if we solve 
\[
    \max_{\pi \in \Pi} \min_{1 \leq k \leq K}  J_k(\pi), ~~~\mbox{where}~~ K \geq 2,
\]
then the max-min solution has the property $J_1(\pi)=J_2(\pi)=\cdots=J_K(\pi)$ if this equalization point is on the Pareto boundary.  In the case of $K=2$, the max-min point is thus the point on which the Pareto boundary meets the line $J_1 = J_2$, as seen in Fig. 1 of the paper if the Pareto boundary and the line $J_1 = J_2$ meet.

Now, suppose that we scale each objective using $\alpha_k > 0, 1 \leq k \leq K$, and solve
\[ 
    \max_{\pi \in \Pi} \min_{1 \leq k \leq K} \alpha_k  J_k(\pi), ~~~\mbox{where}~~ K \geq 2.
\]
This new problem can also be solved by our method by scaling the reward with factor $\alpha_k$.  Then, the max-min solution of the new problem satisfies the new equalizer rule $\alpha_1 J_1 = \alpha_2 J_2 = \cdots = \alpha_K J_K$ if this equalization point is on the Pareto boundary.  In the case of $K=2$,  the new solution is the point on which the Pareto boundary meets the line $\alpha_1 J_1 = \alpha_2 J_2$, i.e., $J_2 =\frac{\alpha_1}{\alpha_2} J_1$, as seen in Fig. 1 of the paper.  Hence, if we want to obtain the Pareto boundary of the problem, then we can sweep the scaling factors $(\alpha_1,\cdots,\alpha_K)$ and solve the max-min problem for each scaling factor set. Then, we can approximately construct the Pareto boundary by  interpolating the points of considered scaling factor sets. 

Please note that there exist  cases that the Pareto boundary and the equalization line $J_1 = J_2=\cdots=J_K$ or $\alpha_1 J_1 = \alpha_2 J_2 = \cdots = \alpha_K J_K$ do not meet. An example is shown in Fig. \ref{fig:four_room_main_body}, the Four-Room environment in Section \ref{subsec:performance}. Then, the above argument may not hold.  However, even in the case of Four-Room where there is no equalizing Pareto-optimal point, we observe that the proposed method nearly achieves the unique max-min Pareto optimal point.

\newpage
\section{Dual Transformation from P0 to P1} \label{append:a_duality}

\begin{proof}

Using additional slack variable $c = \min_{1 \leq k \leq K}  \sum_{s,a} d(s,a) r^{(k)}(s,a) $ to convert \textbf{P0} to the corresponding LP \textbf{P0-LP}, we have:
\begin{equation}
    \textbf{P0-LP}:    \max_{d(s,a), c} c
\end{equation}
\begin{equation}
    \sum_{s,a} r^{(k)}(s,a) d(s,a) \geq c, ~ 1 \leq k \leq K
\end{equation}
\begin{equation}
    \sum_{a'} d(s',a') = \mu_0(s') + \gamma \sum_{s,a} P(s' | s,a) d(s,a) ~~ \forall s'
\end{equation}
\begin{equation}
    d(s,a) \geq 0, ~~ \forall (s,a) .
\end{equation}

We use the following duality transformation in LP: $\max_x u^T x$ s.t. $Ax = b, x \geq 0 \iff \min_y b^T y$ s.t. $A^T y \geq u$.

Inserting additional non-negative variables $\delta_k (k=1, \cdots, K), c^+, c^-$ to change inequality to equality gives
\begin{equation}
    \max_{d(s,a), \delta_k, c^+, c^- } c^+ - c^-; ~ c^+, c^- \geq 0,
\end{equation}
\begin{equation}
     \delta_k = \sum_{s,a} r^{(k)}(s,a) d(s,a) - c^+ + c^-, \delta_k \geq 0, ~ 1 \leq k \leq K
\end{equation}
\begin{equation}
    \sum_{a'} d(s',a') = \mu_0(s') + \gamma \sum_{s,a} P(s' | s,a) d(s,a) ~ \forall s'; ~~ d(s,a) \geq 0, ~ \forall (s,a) .
\end{equation}

Let $|S| = p, |A|=q$. 
The corresponding matrix formulation is the form of $\max_x u^T x$ s.t. $Ax = b, x \geq 0$ where

\begin{equation}
    x = 
    \begin{bmatrix}
        d(s_1, a_1) \\
        \vdots \\
        d(s_{p}, a_{q}) \\
        \delta_1 \\
        \vdots \\
        \delta_K \\
        c^+ \\
        c^- 
    \end{bmatrix}
    \in \mathbb{R}^{pq+K+2}, 
    ~~ u = 
    \begin{bmatrix}
        0 \\
        \vdots \\
        0 \\
        0 \\
        \vdots \\
        0 \\
        1 \\
        -1 
    \end{bmatrix}
    \in \mathbb{R}^{pq+K+2}, 
    ~~ b = 
    \begin{bmatrix}
        \mu_0(s_1) \\
        \vdots \\
        \mu_0(s_{p}) \\
        0 \\
        \vdots \\
        0
    \end{bmatrix}
    \in \mathbb{R}^{p+K}, 
\end{equation}

\begin{equation}
    A = 
    \begin{bmatrix}
        \begin{array}{c|c|c|c|c}
            A_1 & A_2 & \cdots & A_{p} & D 
        \end{array}
    \end{bmatrix}
    \in \mathbb{R}^{(p+K) \times (pq + K + 2)} ~ \text{with} ~ A_i \in \mathbb{R}^{(p+K) \times q} (1 \leq i \leq p), D \in \mathbb{R}^{(p+K) \times (K+2)}
\end{equation}
where

\[ [A_i]_{jk} =
  \begin{cases}
    1-\gamma P(s_j |s_i, a_k) = 1-\gamma P(s_i |s_i, a_k) & \text{if } j = i \\
    -\gamma P(s_j |s_i, a_k) & \text{if } j \neq i ~ \text{and} ~  1 \leq j \leq p  \\ 
    r^{(j-p)}(s_i, a_k) & \text{if } p+1 \leq j \leq p+K 
  \end{cases}
\]
and
\begin{equation}
    D = 
    \begin{bmatrix}
        \begin{array}{c|c}
            O_{p\times K} & O_{p\times 2} \\  
            \hline
            -I_K & -1_K | 1_K \\
        \end{array}
    \end{bmatrix}
    \in \mathbb{R}^{(p+K) \times (K+2)}.
\end{equation}
Here, $1_K$ is the all-one column vector of length $K$.

Let $y = [v(s_1), \cdots, v(s_{p}), w_1, \cdots, w_K]^T \in \mathbb{R}^{p+K}$. The dual LP problem $\min_y b^T y$ s.t. $A^T y \geq u$ is 
\begin{equation} 
    \min_{w, v} \sum_s \mu_0(s) v(s)
\end{equation}
\begin{equation} 
     v(s) - \gamma \sum_{s'}P(s'|s,a)v(s') + \sum_{k=1}^K w_k r^{(k)}(s,a) \geq 0,~\forall (s,a)
\end{equation}
\begin{equation}
    -w_k \geq 0, ~ 1 \leq k \leq K,
\end{equation}
\begin{equation} 
    -\sum_{k=1}^K w_k \geq 1, ~ \sum_{k=1}^K w_k \geq -1.
\end{equation}

Note that we have the equality constraint of $-\sum_{k=1}^K w_k = 1$. Changing notation from $-w_k$ to $w_k$ gives the following \textbf{P1} problem:
\begin{equation} 
    \min_{w, v} \sum_s \mu_0(s) v(s)
\end{equation}
\begin{equation} 
     \forall (s,a) , ~ v(s) \geq \sum_{k=1}^K w_k r^{(k)}(s,a) + \gamma \sum_{s'}P(s'|s,a)v(s') 
\end{equation}
\begin{equation} 
    \sum_{k=1}^K w_k = 1; ~ w_k \geq 0 ~~ \forall 1 \leq k \leq K.
\end{equation}

\end{proof}

\newpage
\section{Proof of Convexity in Value Iteration} \label{append:convexity}

Recall the Bellman optimality operator $T^*_{w}$ given a weight vector $w \in \mathbb{R}^K$:
\begin{equation}
    \forall s,~~  (T^*_{w} v)(s) := \max_a \left[ \sum_{k=1}^K w_k r^{(k)}(s,a) + \gamma \sum_{s'} P(s'|s,a) v(s') \right].
\end{equation}
Let the unique converged result of the mapping $T^*_{w}$ be $v^{*}_{w}$ which is a function of $w$. We first show that $v^{*}_{w}(s), \forall s,$ is a convex function for $w$. Then due to the linearity, the objective $\mathcal{L}(w) = \sum_s \mu_0(s) v^{*}_{w}(s)$ is also convex for $w$. 

\begin{proof}

For $0 \leq \lambda \leq 1$ and $w^1, w^2 \in \mathbb{R}^K$, let $\bar{w}_{\lambda} := \lambda w^1 + (1-\lambda) w^2$, and set $v$ arbitrary. We will show that for any positive integer $p \geq 1$,
\begin{equation}
    (T^*_{\bar{w}_{\lambda}})^p v \leq \lambda (T^*_{w^1})^p v + (1-\lambda)(T^*_{w^2})^p v.
\end{equation}
If we let $p \rightarrow \infty$, then $v^*_{\bar{w}_{\lambda}}(s) \leq \lambda v^*_{w^1}(s) + (1-\lambda) v^*_{w^2}(s), ~ \forall s$, and the proof is done. We use induction as follows.

Step 1. Base case. Let $a^0_*(s) := \arg \max_a \left[ \sum_k \{ \lambda w_k^{1} + (1-\lambda) w_k^{2} \} r^{(k)}(s,a) + \gamma \sum_{s'} P(s' | s,a) v(s')  \right]$. Then
\begin{align}
     \forall s, [T^*_{\bar{w}_{\lambda}} v](s) &= \max_a \left[ \sum_k \{ \lambda w_k^{1} + (1-\lambda) w_k^{2} \} r^{(k)}(s,a) + \gamma \sum_{s'} P(s' | s,a) v(s')  \right] \nonumber \\
     &= \lambda \left[ \sum_k  w_k^{1}  r^{(k)}(s,a^0_*(s)) + \gamma \sum_{s'} P(s' | s,a^0_*(s)) v(s')  \right] \nonumber \\
     &+ (1-\lambda) \left[ \sum_k  w_k^{2}  r^{(k)}(s,a^0_*(s)) + \gamma \sum_{s'} P(s' | s,a^0_*(s)) v(s')  \right] \nonumber \\
     &\leq  \lambda [T^*_{w^1} v](s) + (1 - \lambda) [T^*_{w^2} v](s).
\end{align}

Step 2. Assume the following is satisfied for a positive integer $p \geq 1$:
\begin{equation} \label{eq:convexity_proof_step2_assumption}
    (T^*_{\bar{w}_{\lambda}})^p v \leq \lambda (T^*_{w^1})^p v + (1-\lambda) (T^*_{w^2})^p v.
\end{equation}
Let $a^p_*(s) := \arg \max_a [ \sum_k \{ \lambda w_k^{1} + (1-\lambda) w_k^{2} \} r^{(k)}(s,a) + \gamma \sum_{s'} P(s' | s,a) [ \lambda (T^*_{w^1})^p v + (1-\lambda)  (T^*_{w^2})^p v](s')  ]$. Then
\begin{align}
     \forall s \in S, ~ [ (T^*_{\bar{w}_{\lambda}})^{p+1} v](s) 
     % &\leq [ (T^*_{\bar{w}_{\lambda}}) \{ \lambda (T^*_{w^1})^p + (T^*_{w^2})^p \} v](s)  \nonumber \\
     &= \max_a \left[ \sum_k \{ \lambda w_k^{1} + (1-\lambda) w_k^{2} \} r^{(k)}(s,a) + \gamma \sum_{s'} P(s' | s,a) [ (T^*_{\bar{w}_{\lambda}})^{p} v](s')  \right] \nonumber \\
     &\leq \max_a [ \sum_k \{ \lambda w_k^{1} + (1-\lambda) w_k^{2} \} r^{(k)}(s,a) \nonumber \\
     &+ \gamma \sum_{s'} P(s' | s,a) [ \lambda (T^*_{w^1})^p v + (1-\lambda)  (T^*_{w^2})^p v](s')  ] ~ (\text{Use} ~   \eqref{eq:convexity_proof_step2_assumption}) \nonumber \\
     &= \lambda \left[ \sum_k  w_k^{1}  r^{(k)}(s,a^p_*(s)) + \gamma \sum_{s'} P(s' | s,a^p_*(s)) [(T^*_{w^1})^p v](s')  \right] \nonumber \\
     &+ (1-\lambda) \left[ \sum_k  w_k^{2}  r^{(k)}(s,a^p_*(s)) + \gamma \sum_{s'} P(s' | s,a^p_*(s)) [(T^*_{w^2})^p v](s')  \right] \nonumber \\
     &\leq  \lambda [(T^*_{w^1})^{p+1} v](s) + (1 - \lambda) [(T^*_{w^2})^{p+1} v](s).
\end{align}

\end{proof}

\newpage
\section{Proof of Piecewise-linearity} \label{append:piecewise_linear}

\begin{lemma}\label{lemma:Piecewise-linearity}
Let $A$ be a row stochastic square matrix. Then for any $\gamma\in[0,1)$, $I-\gamma A$ is invertible where $I$ is identity matrix with the same size \cite{horn2012matrix}.
\end{lemma}

\begin{proof}
    Let $A\in\mathbb{R}^{n\times n}$ and $a_i^T\in\mathbb{R}^{n}$ be $i$-th row of $A$.\\
    We show that $x(\in\mathbb{R}^n)\neq0\implies(I-\gamma A)x\neq0$, which is equivalent to ensuring that $I-\gamma A$ is invertible.
    \begin{align*}
        ||(I-\gamma A)x||_\infty & =||x-\gamma Ax||_\infty\\
        & \ge||x||_\infty-\gamma||Ax||_\infty~(\because \text{triangular inequality})\\
        & =||x||_\infty-\gamma\max_i\{|a_i^Tx|\}\\
        & \ge||x||_\infty-\gamma\max_i\{||a_i||_1 ||x||_\infty\}~(\because \text{H\"{o}lder inequality for each }i)\\
        & = ||x||_\infty-\gamma||x||_\infty~(\because \text{row sum 1 with non-negative elements})\\
        & = (1-\gamma)||x||_\infty>0~(\because\gamma<1,~x\neq0)
    \end{align*}
\end{proof}

\textbf{Theorem \ref{thm_3}}
Let the state space $S$ and action space $A$ are finite. Then for any $s\in S$, $v^{*}_{w}(s)$ is a  piecewise-linear function with respect to $w \in \mathbb{R}^K$. Consequently. the objective $\mathcal{L}(w) = \sum_s \mu_0(s) v^{*}_{w}(s)$ is also  piecewise-linear with respect to $w \in \mathbb{R}^K$. 

\begin{proof} Let $S=\{s_1,\dots,s_{p}\}$ and $A=\{a_1,\dots,a_{q}\}$. 
Recall the Bellman optimality operator $T^*_{w}$ given a weight vector $w \in \mathbb{R}^K$:
\begin{equation}
    \forall s,~~  (T^*_{w} v)(s) := \max_a \left[ \sum_{k=1}^K w_k r^{(k)}(s,a) + \gamma \sum_{s'} P(s'|s,a) v(s') \right].
\end{equation}

 By the theory of (single objective) MDP \cite{Puterman_2005}, $<S,A,P,\mu_0,\sum_{k=1}^K w_k r^{(k)},\gamma>$ which is an MDP induced by any $w\in \mathbb{R}^K$ has the unique optimal value function $v^*_w$ and $v^*_w=T^*_wv^*_w$ holds, i.e.
\begin{equation}\label{eq:piecewise_linearity_proof_bellman_operator}
    v^*_w(s) = \max_{a\in A}\{\sum_{k=1}^K w_k r^{(k)}(s,a) + \gamma \sum_{s'}P(s'|s,a)v^*_w(s')\}~\forall w\in\mathbb{R}^K,~s\in S
\end{equation}

For simplicity, we use vector expression; $r(s,a)=[r^{(1)}(s,a),\ldots,r^{(K)}(s,a)]^T\in\mathbb{R}^K$. 

For each $s\in S$, let $D_i(s):=\{w\in\mathbb{R}^K|~i = \min\{argmax_j\{r(s,a_j)^T w + \gamma\sum_{s'}P(s'|s,a_j)v^*_w(s')\}\}$, then $Part(s):=\{D_1(s),\dots,D_{q}(s)\}$ is a partition of $\mathbb{R}^K$ for each $s\in S$. In other words, for arbitrary given $s\in S$, $w\in D_i(s)$ if $a_i$ maximizes RHS of \eqref{eq:piecewise_linearity_proof_bellman_operator} with minimal index $i$. Note that since $A$ is a finite set, minimum operator in $D_i(s)$ is well-defined.\\

For each $i\in[q]:=\{1,\ldots,q\}$, by definition of $D_i(s)$, $v^*_w(s) = r(s,a_i)^T w + \gamma \sum_{s'}P(s'|s,a_i)v^*_w(s')~\forall w\in D_i(s)$.\\
We take the refinement of all partitions $Part(s)$, i.e., $\{D_{i_1}(s_1)\bigcap\cdots\bigcap D_{i_{p}}(s_{p})|~i_j\in[q]~\forall j\}$ which is a partition of $\Delta^K$ consists of at most $q^{p}$ subsets of $\mathbb{R}^K$.\\

On each non-empty $D_{i_1}(s_1)\bigcap\cdots\bigcap D_{i_{p}}(s_{p})$ ($i_j\in[q]~\forall j$),
\begin{align}
    &v^*_w(s_1) = r(s_1,a_{i_1})^Tw+\gamma\Sigma_{s'}P(s'|s_1,a_{i_1})v^*_w(s') \nonumber\\
    &\vdots \nonumber\\
    &v^*_w(s_{p}) = r(s_{p},a_{i_{p}})^Tw+\gamma\Sigma_{s'}P(s'|s_{p},a_{i_{p}})v^*_w(s') \nonumber\\
    \implies & \begin{bmatrix}
        v^*_w(s_1) \\
        \vdots \\
        v^*_w(s_{p})
    \end{bmatrix} = 
    \begin{bmatrix}
        r(s_1,a_{i_1})^T \\
        \vdots \\
        r(s_{p},a_{i_{p}})^T
    \end{bmatrix}w + \gamma
    \begin{bmatrix}
        P(s_1|s_1,a_{i_1}) & \cdots & P(s_{p}|s_1,a_{i_1}) \\
        \vdots & \ddots & \vdots \\
        P(s_1|s_{p},a_{i_{p}}) & \cdots & P(s_{p}|s_{p},a_{i_{p}})
    \end{bmatrix}
    \begin{bmatrix}
        v^*_w(s_1) \\
        \vdots \\
        v^*_w(s_{p})
    \end{bmatrix}\label{eq:piecewise_linearity_proof_matrix_form}
\end{align}
Let $R(i_1,\dots,i_{p})=\begin{bmatrix}
        r(s_1,a_{i_1})^T \\
        \vdots \\
        r(s_{p},a_{i_{p}})^T
    \end{bmatrix}$ and $B(i_1,\dots,i_{p})=\begin{bmatrix}
        P(s_1|s_1,a_{i_1}) & \cdots & P(s_{p}|s_1,a_{i_1}) \\
        \vdots & \ddots & \vdots \\
        P(s_1|s_{p},a_{i_{p}}) & \cdots & P(s_{p}|s_{p},a_{i_{p}})
    \end{bmatrix}$,\\
    which are constant of $w$.\\
    
    Note that $B(i_1,\dots,i_{p})$ has non-negative elements and all row sums are 1. By lemma \ref{lemma:Piecewise-linearity}, $I-\gamma B(i_1,\dots,i_{p})$ is invertible. From \eqref{eq:piecewise_linearity_proof_matrix_form}, 
    $v_w^* = [(I-\gamma B(i_1,\dots,i_{p}))^{-1}R(i_1,\dots,i_{p})]w~~\forall w\in D_{i_1}(s_1)\bigcap\cdots\bigcap D_{i_{p}}(s_{p})$ and thus, $v_w^*$ is linear on each non-empty $D_{i_1}(s_1)\bigcap\cdots\bigcap D_{i_{p}}(s_{p})$. Therefore, $v_w^*$ is piecewise-linear on $\mathbb{R}^K$ with at most ${q}^{p}$ pieces.
\end{proof}

\newpage
\section{Comparison between Entropy Regularization and KL-Divergence based Regularization} \label{append:entropy_regul_simple}

In addition to ensuring convex optimization and promoting exploration, entropy regularization is favored over  general KL-divergence counterpart due to its \textbf{algorithmic simplicity}.

First, we present the following KL-divergence regularized formulation, denoted as \textbf{P0'-KL}, and its convex dual problem, denoted as \textbf{P1'-KL}:

\begin{align*}
    \textbf{P0'-KL}:~~\max_d\min_{1\le k \le K}~& \sum_{s,a} d(s,a)(r^{(k)}(s,a) - \alpha D(\pi^d(\cdot|s) || \pi^{d_\beta}(\cdot|s) ) )\\
    s.t.~~&\sum_{a'}d(s',a') = \mu_0(s') + \gamma\sum_{s,a}P(s'|s,a)d(s,a)~\forall s'\\
    & d(s,a)\ge 0 ~\forall s,a
\end{align*}
where $\pi^d(a|s):=\frac{d(s,a)}{\sum_{a'}d(s,a')}$;  $\pi^{d_\beta}(a|s):=\frac{d_\beta(s,a)}{\sum_{a'}d_\beta(s,a')}$ is the anchor policy from any anchor distribution we want  $d_\beta:S\times A\to\mathbb{R}$; and $D(\cdot||\cdot)$ denotes KL-divergence. Assume that $d_\beta(s,a)>0~\forall s,a$. Using the similar manipulation in Section \ref{subsec:reg_maxent_derive}, the dual problem reduces to the following problem:
\begin{align*}
    \textbf{P1'-KL}: ~~ &\min_{w \in \Delta^K, v} \sum_s \mu_0(s) v(s)\\
    s.t.~~&v(s)  = \alpha \log \sum_a \pi^{d_\beta}(a|s) \exp [ \frac{1}{\alpha}  \{ \sum_{k=1}^K w_k r^{(k)}(s,a) + \gamma \sum_{s'}P(s'|s,a)v(s') \} ].
\end{align*}

In general, additional processes are required for learning or memorizing $\pi^{d_\beta}$ to impose specific target or anchor information we want. For example, in offline RL setting $\pi^{d_\beta}$ is learned to follow the behavior policy that generated pre-collected data \cite{wu19brac,kumar20cql}. Please note that entropy regularization corresponds to the special case of KL regularization in which the anchor  distribution $\pi^{d_\beta}$ is simply uniform. Consequently, there is no need for additional learning procedures or memory regarding $\pi^{d_\beta}$, and the problem becomes simpler because $\pi^{d_\beta}$ is uniform in the above equation.

\newpage
\section{Solution of P1' in the One-state Example} \label{append:reg_sol_one_state}

\textbf{P1'} is written as follows:
\begin{equation}
        \min_{v(s_1), w_1} v(s_1) 
\end{equation}
\begin{equation}
     \exp(\frac{3 w_1 - (1-\gamma)v(s_1)}{\alpha})  + \exp(\frac{3 (1 - w_1) - (1-\gamma)v(s_1)}{\alpha}) + \exp(\frac{1 - (1-\gamma)v(s_1)}{\alpha}) = 1.
\end{equation}
\begin{equation}
    0 \leq w_1 \leq 1.
\end{equation}
This is equivalent to
\begin{equation}
    \min_{0 \leq w_1 \leq 1} v(s_1) = \frac{\alpha}{1-\gamma} \log \bigg[ \exp(\frac{3 w_1}{\alpha}) + \exp(\frac{3 (1-w_1)}{\alpha}) + \exp(\frac{1}{\alpha}) \bigg].
\end{equation}

\begin{itemize}
    \item The analytic exact solution is $w^*_1 = w^*_2 = \frac{1}{2}, v^*(s_1) = \frac{\alpha}{1-\gamma} \log \bigg[  2\exp(\frac{3}{2 \alpha}) + \exp(\frac{1}{\alpha}) \bigg]$.
    \item $\pi^*(a_1 | s_1) = \pi^*(a_2 | s_1) = \frac{1}{2 + \exp(-\frac{1}{2 \alpha})}, \pi^*(a_3 | s_1) = \frac{\exp(-\frac{1}{2 \alpha})}{2 + \exp(-\frac{1}{2 \alpha})}$.
    \item $\alpha \rightarrow 0^+$ recovers the max-min optimal policy $\pi^*(a_1 | s_1) = \pi^*(a_2 | s_1) = 0.5$ in \textbf{P0}.
\end{itemize}

We denote the optimal value for the regularized problem as $v^*_\alpha(s_1) := \frac{\alpha}{1-\gamma} \log \bigg[  2\exp(\frac{3}{2 \alpha}) + \exp(\frac{1}{\alpha}) \bigg]$. Then, the gap between  $v^*(s_1)$, the optimal value for \textbf{P1}, and $v^*_\alpha(s_1)$ is 
 \[
 |v^*(s_1)-v^*_\alpha(s_1)| 
= \frac{1}{1-\gamma}|\alpha\log2 + \alpha\log(1+\frac{1}{2}\exp(-\frac{1}{2\alpha}))|
\le \frac{1}{1-\gamma}|\alpha\log2 + \frac{\alpha}{2}\exp(-\frac{1}{2\alpha})| = O(\alpha)
\]
with $\alpha>0$  $(\because \log(1+t) \le t)$. The gap vanishes as $\alpha\to0$ in the one-state two-objective MDP example. 

As mentioned in Appendix \ref{subappend:implement}, we scheduled $\alpha$ during training so that it diminishes as time goes on.  Hence, in the later stage of learning, we expect the gap to diminish in our implemented algorithm.

\newpage
\section{Proof of Convexity in Soft Value Iteration} \label{append:convexity_soft}

\textbf{Theorem \ref{thm_4}}
Let $(\mathcal{T}^{soft,*}_{w} v)(s) := \alpha \log \sum_a \exp [ 1/\alpha * \{ \sum_{k=1}^K w_k r^{(k)}(s,a) + \gamma \sum_{s'}P(s'|s,a)v(s') \} ]$, $\forall s \in S$. Let the unique fixed point of $\mathcal{T}^{soft,*}_{w}$ be $v^{soft,*}_{w}$, and $\mathcal{L}^{soft}(w) := \sum_s \mu_0(s) v^{soft,*}_{w}(s)$. Then $v^{soft,*}_{w}(s), \forall s \in S,$ is a convex function with respect to $w \in \mathbb{R}^K$. Consequently, the objective $\mathcal{L}^{soft}(w) = \sum_s \mu_0(s) v^{soft,*}_{w}(s)$ is also convex with respect to $w \in \mathbb{R}^K$. 

\begin{proof}
    We first show that $v^{soft,*}_{w}(s), \forall s,$ is a convex function for $w$. Then due to the linearity, the objective $\mathcal{L}^{soft}(w) = \sum_s \mu_0(s) v^{soft,*}_{w}(s)$ is also convex for $w$.

    For $0 \leq \lambda \leq 1$ and $w^1, w^2 \in \mathbb{R}^K$, let $\bar{w}_{\lambda} := \lambda w^1 + (1-\lambda) w^2$, and set $v$ arbitrary. We will show that for any positive integer $p \geq 1$,
    \begin{equation}
        (\mathcal{T}^{soft,*}_{\bar{w}_{\lambda}})^p v \leq \lambda (\mathcal{T}^{soft,*}_{w^1})^p v + (1-\lambda)(\mathcal{T}^{soft,*}_{w^2})^p v.
    \end{equation}
    If we let $p \rightarrow \infty$, then $v^{soft,*}_{\bar{w}_{\lambda}}(s) \leq \lambda v^{soft,*}_{w^1}(s) + (1-\lambda) v^{soft,*}_{w^2}(s), ~ \forall s$, and the proof is done. If $\lambda=0$ or $1$, equality is satisfied for $p \geq 1$. Suppose $0 < \lambda < 1$. We use induction as follows.
    
Step 1. Base case. According to H\"{o}lder's inequality, we have 
\begin{equation} \label{eq:holder}
    \log \sum_{a} u_a^\lambda v_a^{1-\lambda} \leq \log \bigg[ \left\{ \sum_{a} (u_a^\lambda)^{\frac{1}{\lambda}} \right\}^\lambda \cdot \left\{ \sum_{a} (v_a^{1-\lambda})^{\frac{1}{1-\lambda}} \right\}^{1-\lambda} \bigg] =  \lambda \log \sum_{a} u_a + (1 - \lambda)\log \sum_{a} v_a.
\end{equation}

Setting
\begin{equation}
    u_a = \exp \left\{ 1/\alpha * \{ \sum_{k=1}^K w_k^1 r^{(k)}(s,a) + \gamma \sum_{s'}P(s'|s,a)v(s') \}  \right\}
\end{equation}
and 
\begin{equation}
    v_a = \exp \left\{ 1/\alpha * \{ \sum_{k=1}^K w_k^2 r^{(k)}(s,a) + \gamma \sum_{s'}P(s'|s,a)v(s') \}  \right\}
\end{equation}
directly gives
\begin{equation}
    [\mathcal{T}^{soft,*}_{\bar{w}_{\lambda}} v](s) \leq  \lambda [\mathcal{T}^{soft,*}_{w^1} v](s) + (1 - \lambda) [\mathcal{T}^{soft,*}_{w^2} v](s), \forall s \in S.
\end{equation}

Step 2. Assume the following is satisfied for a positive integer $p \geq 1$:
\begin{equation} \label{eq:convexity_proof_step2_assumption_soft}
    (\mathcal{T}^{soft,*}_{\bar{w}_{\lambda}})^p v \leq \lambda (\mathcal{T}^{soft,*}_{w^1})^p v + (1-\lambda) (\mathcal{T}^{soft,*}_{w^2})^p v.
\end{equation}
Then $\forall s \in S$, we have 
\begin{align}
     & [ (\mathcal{T}^{soft,*}_{\bar{w}_{\lambda}})^{p+1} v](s) \nonumber \\
     &= \alpha \log \sum_a \exp \left[ 1/\alpha * \sum_k \{ \lambda w_k^{1} + (1-\lambda) w_k^{2} \} r^{(k)}(s,a) + 1/\alpha *\gamma \sum_{s'} P(s' | s,a) [ (\mathcal{T}^{soft,*}_{\bar{w}_{\lambda}})^{p} v](s')  \right] \nonumber \\
     &\leq  \alpha \log \sum_a \exp [ 1/\alpha * \sum_k \{ \lambda w_k^{1} + (1-\lambda) w_k^{2} \} r^{(k)}(s,a) \nonumber \\
     &+ 1/\alpha * \gamma \sum_{s'} P(s' | s,a) [ \lambda (\mathcal{T}^{soft,*}_{w^1})^p v + (1-\lambda)  (\mathcal{T}^{soft,*}_{w^2})^p v](s')  ] ~ (\text{Use} ~   \eqref{eq:convexity_proof_step2_assumption_soft}) \nonumber \\
     &\leq  \lambda [(\mathcal{T}^{soft,*}_{w^1})^{p+1} v](s) + (1 - \lambda) [(\mathcal{T}^{soft,*}_{w^2})^{p+1} v](s).
\end{align}
The last inequality is directly given from $u_a = \exp \left\{ 1/\alpha * \{ \sum_{k=1}^K w_k^1 r^{(k)}(s,a) + \gamma \sum_{s'}P(s'|s,a)[(\mathcal{T}^{soft,*}_{w^1})^{p} v](s') \}  \right\}$ and $v_a = \exp \left\{ 1/\alpha * \{ \sum_{k=1}^K w_k^2 r^{(k)}(s,a) + \gamma \sum_{s'}P(s'|s,a)[(\mathcal{T}^{soft,*}_{w^2})^{p} v](s') \}  \right\}$, and applying \eqref{eq:holder}.

\end{proof}

\newpage
\section{ Proof of Continuous Differentiability of $v_w^{soft,*}$ } \label{append:continuous_differentiable}

\textbf{Theorem \ref{thm_6}} $v_w^{soft,*}$ is continuously differentiable in $w$ on $\mathbb{R}^K$.

\begin{proof}
    Let $|S|=p$. Define $f(w,v):=\mathcal{T}^{soft,*}_w v\in\mathbb{R}^{|S|}$, $F(w,v):=v-f(w,v) \in\mathbb{R}^{|S|}$, and let $w\in\mathbb{R}^K$ be arbitrary fixed. Since $\mathcal{T}^{soft,*}_w$ is a contraction mapping for each $w\in\mathbb{R}^K$, by Banach fixed point theorem, there exists unique fixed point $v^{soft,*}_w$ for each $w$. It means that $v^{soft,*}_w=f(w,v^{soft,*}_w)~\forall w$, which is equivalent to $F(w,v^{soft,*}_w)=0\in\mathbb{R}^{|S|}~\forall w$. For the proof, we will apply implicit function theorem to $F$.\\

    First, $f$ is continuously differentiable in $(w,v)$ since it is a composition of logarithm, summation, exponential and linear functions. Therefore, $F(w,v)$ is continuously differentiable since $v$ and $f(w,v)$ are continuously differentiable.  -(a)\\

    Next, we check the condition for implicit function theorem that the $p \times p$ Jacobian matrix $\partial_v F(w,v^{soft,*}_w) := \partial_v F(w,v)|_{v = v^{soft,*}_w}$ is invertible where $\partial_v F(w,v)$ is a matrix $\begin{bmatrix}
    \partial_v (F(w,v)(s_1))\\
    \vdots\\
    \partial_v (F(w,v)(s_p))
\end{bmatrix}\in\mathbb{R}^{p\times p}$.
    
    We have
    \begin{align*}
        \partial_v F(w,v) &= I-\partial_v f(w,v).
    \end{align*}
    From $f(w,v)(s)= \alpha\log\sum_a\exp[\frac{1}{\alpha}(r(s,a)^T w + \gamma\sum_{s'}P(s'|s,a)v(s'))]$, The $s$-th row of Jacobian $\partial_v f(w,v)$ is 
    \begin{align*}
        \partial_v (f(w,v)(s))^T &= \gamma\sum_a \beta_{s,w,v}(a)[P(s_1|s,a),\ldots, P(s_p|s,a)]\\
        \text{where } \beta_{s,w,v}(a)&= 
        %\frac{\exp\frac{1}{\alpha}(r(s,a)^T w + \gamma\Sigma_{s'}P(s'|s,a)v(s'))}{\Sigma_a\exp\frac{1}{\alpha}(r(s,a)^T w + \gamma\Sigma_{s'}P(s'|s,a)v(s'))}\\
        \text{softmax}(\frac{1}{\alpha}(r(s,\cdot)^T w + \gamma\sum_{s'}P(s'|s,\cdot)v(s')))(a)\\
        &= \exp[\frac{1}{\alpha}(r(s,a)^T w + \gamma\sum_{s'}P(s'|s,a)v(s'))] / \sum_{a'}\exp[\frac{1}{\alpha}(r(s,a')^T w + \gamma\sum_{s'}P(s'|s,a')v(s'))].
    \end{align*}
    Denote the transition probability vector $[P(s_1|s,a)\cdots P(s_p|s,a)]$ as $P(\cdot|s,a)$.

    Note that $\sum_a\beta_{s,w,v}(a)=1~\forall s,w,v$. Thus, the sum of elements in the $s$-th row of Jacobian $\partial_v f(w,v)$ (i.e. $\gamma \sum_a \beta_{s,w,v}(a)P(\cdot|s,a)$) is $\gamma \sum_{s'}\sum_a \beta_{s,w,v}(a)P(s'|s,a) = \gamma \sum_a \beta_{s,w,v}(a)\sum_{s'}P(s'|s,a)=\gamma,~\forall s,w,v$.\\
    Then, $\partial_v F(w,v^{soft,*}_w) = I-\partial_v f(w,v^{soft,*}_w)=I-\gamma\begin{bmatrix}
        \sum_a \beta_{s_1,w,v^{soft,*}_w}(a)P(\cdot|s_1,a)\\
        \vdots\\
        \sum_a \beta_{s_p,w,v^{soft,*}_w}(a)P(\cdot|s_p,a)
    \end{bmatrix}$ is invertible by Lemma \ref{lemma:Piecewise-linearity}.  -(b)\\

    From (a) and (b), by implicit function theorem, for each $w\in \mathbb{R}^K$ there exists an open set $U\subset\mathbb{R}^K$ containing $w$ such that there exists a unique continuously differentiable function $g:U\to\mathbb{R}^{|S|}$ such that $g(w)=v^{soft,*}_w$ and $F(w',g(w'))=0$, i.e., $g(w')=f(w',g(w'))$ for all $w'\in U$. It means that $g(w')$ is a fixed point of $f(w',\cdot) = \mathcal{T}^{soft,*}_{w'}$ for any $w'\in U$. 

    Since the fixed point of $\mathcal{T}^{soft,*}_{w'}$ is unique, $g(w') = v^{soft,*}_{w'}~\forall w'\in U$. Therefore, $v^{soft,*}_{w'}$ is continuously differentiable in $w'\in U$. Recall that we acquired $g = g_w$ and $U = U_w$ from a given $w \in \mathbb{R}^K$. If we analogously apply this logic for all $w \in \mathbb{R}^K$, we have $g_w(\cdot) = v^{soft,*}_{(\cdot)}$ in $U_w$. Since each $g_w$ is continuously differentiable in $U_w$ and
    $\bigcup_{w}U_w = \mathbb{R}^K$, $v^{soft,*}_{(\cdot)}$ is continuously differentiable on $\mathbb{R}^K$.
\end{proof}

\newpage
\section{Proof of Lipschitz Continuity of $v_w^{soft,*}$} \label{append:lipschitz_continuity}
\subsection{Soft Bellman Operator for Given $w\in\mathbb{R}^K$}\label{subsec:soft_bellman_operator_definition}
Let $|S|=p$. Define the Soft Bellman operator $\mathcal{T}^{soft,*}_w$ for MDP induced by $w\in\mathbb{R}^K$ as follows:
\begin{align}
    &\mathcal{T}^{soft,*}_w:\mathbb{R}^{|S|}\to\mathbb{R}^{|S|} \text{ where} \nonumber\\
    &(\mathcal{T}^{soft,*}_wv)(s) = \alpha\log\Sigma_a\exp\frac{1}{\alpha}(r(s,a)^T w + \gamma\Sigma_{s'}P(s'|s,a)v(s'))~\forall s\in S \text{. In vector form,} \nonumber\\
    &\mathcal{T}^{soft,*}_wv = \begin{bmatrix}
        \alpha\log\Sigma_a\exp\frac{1}{\alpha}[\gamma P(\cdot|s_1,a);r(s_1,a)]^T[v;w]\\
        \vdots\\
        \alpha\log\Sigma_a\exp\frac{1}{\alpha}[\gamma P(\cdot|s_p,a);r(s_p,a)]^T[v;w]
    \end{bmatrix}
\end{align}
Note that $[x;y]$ denotes $[x^T,y^T]^T$, vertical concatenation of column vectors $x,y$. From now, define function $f:\mathbb{R}^K\times\mathbb{R}^{|S|}\to\mathbb{R}^{|S|}$ such that $f(w,v):=\mathcal{T}^{soft,*}_wv$, i.e., $f(w,v)(s):=(\mathcal{T}^{soft,*}_wv)(s)~\forall s\in S$.
\subsection{Properties of Soft Bellman Operator}\label{subsec:soft_bellman_operator_properties}
In this subsection, we summarize some properties of soft Bellman operator.

$f$ is continuously differentiable in $(w,v)$ since it is a composition of logarithm, summation, exponential and linear functions. Note that since the term in the logarithm is a sum of exponential which is always positive, derivative of $f$ is continuous in whole domain. In particular, $f(\cdot,v)$ is differentiable for any $v$.

$\mathcal{T}^{soft,*}_w$ is $\gamma$-contraction for all $w\in\mathbb{R}^K$ in $||\cdot||_\infty$. In terms of $f$, $f(w,\cdot)$ is $\gamma$-contraction for all $w\in\mathbb{R}^K$.\\ 
Formally, $||f(w,v_1)-f(w,v_2)||_\infty\le\gamma||v_1-v_2||_\infty~\forall w\in\mathbb{R}^K,v_1,v_2\in\mathbb{R}^{|S|}$. The following is the proof for contraction property that we show for readability of this section. Similar proof is also shown in \citet{fox2016taming,haarnoja2017sql}.

\begin{proof}
    Let $v_1,v_2\in\mathbb{R}^{|S|}$ and $\epsilon=||v_1-v_2||_\infty$, then
    \begin{align*}
        &f(w,v_1)(s)\\
    =~&\alpha\log\Sigma_a\exp\frac{1}{\alpha}(r(s,a)^Tw+\gamma\mathbb{E}_{s'}[v_1(s')])\\
    \le~ &\alpha\log\Sigma_a\exp\frac{1}{\alpha}(r(s,a)^Tw+\gamma\mathbb{E}_{s'}[v_2(s')+\epsilon])\\
    =~&\gamma\epsilon + \alpha\log\Sigma_a\exp\frac{1}{\alpha}(r(s,a)^Tw+\gamma\mathbb{E}_{s'}[v_2(s')])\\
    =~&\gamma\epsilon + f(w,v_2)(s)~\forall s\in S
    \end{align*}
    Analogously, $f(w,v_2)(s)\le\gamma\epsilon + f(w,v_1)(s)~\forall s\in S$.
    Thus, $||f(w,v_1)-f(w,v_2)||_\infty\le\gamma\epsilon=\gamma||v_1-v_2||_\infty$.
\end{proof}

Thus, $\mathcal{T}^{soft,*}_w$ has the unique fixed point by Banach fixed point theorem. Call this fixed point $v^{soft,*}_w$. By the definition, for any fixed $w$, $f(w,v)$ has unique fixed point $v=v^{soft,*}_w$ i.e., $v^{soft,*}_w = f(w,v^{soft,*}_w)$.

Differentiability of $f(\cdot,v)$ and $\gamma$-contraction of $f(w,\cdot)$ are used for the proof of Lipschitz continuity of $v_w^*$ in function of $w$.

\subsection{Proof of Lipschitz Continuity of Soft Bellman Operator}\label{subsec:soft_bellman_operator_proof}
$\partial_w f(w,v)$ is a matrix $\begin{bmatrix}
    \partial_w (f(w,v)(s_1))\\
    \vdots\\
    \partial_w (f(w,v)(s_p))
\end{bmatrix}\in\mathbb{R}^{p\times K}$. We show that its each row is $L_1$-norm bounded by the maximum norm of reward.

\begin{lemma}\label{lemma:soft_bellman_operator_lemma}
    $||\partial_w (f(w,v)(s))||_1\le \max_{a}||r(s,a)||_1$ $\forall s\in S, w\in\mathbb{R}^K,v\in\mathbb{R}^{p}$
\end{lemma}

\begin{proof}
    \begin{align*}
        &||\partial_w f(w,v)(s)||_1\\
        = &||\frac{\partial}{\partial w}\alpha\log\Sigma_a\exp\frac{1}{\alpha}(r(s,a)^Tw+\gamma\Sigma_{s'}P(s'|s,a)v(s'))||_1\\
        = &||\frac{\Sigma_a\exp\frac{1}{\alpha}(r(s,a)^Tw+\gamma\Sigma_{s'}P(s'|s,a)v(s'))\cdot r(s,a)}{\Sigma_a\exp\frac{1}{\alpha}(r(s,a)^Tw+\gamma\Sigma_{s'}P(s'|s,a)v(s'))}||_1
    \end{align*}
    Let $\beta_{s,w,v}(a):=\text{softmax}(\frac{1}{\alpha}(r(s,\cdot)^Tw+\gamma\Sigma_{s'}P(s'|s,\cdot)v(s')))$, then
    \begin{align*}
        ||\partial_w f(w,v)(s)||_1 = &||\Sigma_a \beta_{s,w,v}(a)r(s,a)||_1\\
        \le &\Sigma_a \beta_{s,w,v}(a)||r(s,a)||_1\\
        \le &\max_{a}||r(s,a)||_1~\forall s\in S, w\in\mathbb{R}^K,v\in\mathbb{R}^{|S|}~(\because \Sigma_a\beta_{s,w,v}(a)=1~\forall s,w,v)
    \end{align*}
    Therefore, $||\partial_w f(w,v)(s)||_1\le\max_{a}||r(s,a)||_1~\forall s\in S, w\in\mathbb{R}^K,v\in\mathbb{R}^{p}$.
\end{proof}

\textbf{Theorem \ref{thm_7}}
$v^{soft,*}_w$ is Lipschitz continuous as a function of $w$ on $\mathbb{R}^K$ in $||\cdot||_\infty$.

\begin{proof}
    Let $\epsilon\in\mathbb{R}^K$.
    \begin{align*}
        ||v_{w+\epsilon}^{soft,*}-v_w^{soft,*}||_\infty &= ||f(w+\epsilon,v_{w+\epsilon}^{soft,*})-f(w,v_w^{soft,*})||_\infty \text{ (fixed point)}\\
        &=||f(w+\epsilon,v_{w+\epsilon}^{soft,*})-f(w+\epsilon,v_{w}^{soft,*})+f(w+\epsilon,v_{w}^{soft,*})-f(w,v_w^{soft,*})||_\infty\\
        &\le ||f(w+\epsilon,v_{w+\epsilon}^{soft,*})-f(w+\epsilon,v_{w}^{soft,*})||_\infty + 
        ||f(w+\epsilon,v_{w}^{soft,*})-f(w,v_w^{soft,*})||_\infty\\
        &\le \gamma||v_{w+\epsilon}^{soft,*}-v_{w}^{soft,*}||_\infty + ||\partial_w f(\Tilde{w},v_w^{soft,*})\epsilon||_\infty\text{ for some $\Tilde{w}\in\mathbb{R}^K~~-(*)$ }\\
        &\le \gamma||v_{w+\epsilon}^{soft,*}-v_{w}^{soft,*}||_\infty + \max_{s,a}||r(s,a)||_1||\epsilon||_\infty~~-(**)\\
        \implies &||v_{w+\epsilon}^{soft,*}-v_w^{soft,*}||_\infty\le \frac{\max_{s,a}||r(s,a)||_1}{1-\gamma}||\epsilon||_\infty
    \end{align*}
    In (*), the first term is derived from $\gamma$-contraction of $f(w,\cdot)$, and the second term from Mean Value Theorem under the differentiability of $f(\cdot, v)$.
    Therefore, $v_w^{soft,*}$ is $\frac{\max_{s,a}||r(s,a)||_1}{1-\gamma}$-Lipschitz continuous on $\mathbb{R}^K$. Below is the details for (**).\\
    
    Details for (**):
    \begin{align*}
        &||\partial_w f(\Tilde{w},v_w^{soft,*})\epsilon||_\infty\\
        = &||\begin{bmatrix}
                \partial_w (T^{soft,*}_{\Tilde{w}} v_w^{soft,*}(s_1))^T\epsilon\\
                \vdots\\
                \partial_w (T^{soft,*}_{\Tilde{w}} v_w^{soft,*}(s_p))^T\epsilon
            \end{bmatrix}||_\infty\\
        = &\max_{s}|\partial_w T^{soft,*}_{\Tilde{w}} v_w^{soft,*}(s)^T\epsilon|\\
        \le &\max_{s}||\partial_w T^{soft,*}_{\Tilde{w}} v_w^{soft,*}(s)||_1||\epsilon||_\infty~\text{(H\"{o}lder inequality)}\\
        \le &\max_{s,a}||r(s,a)||_1||\epsilon||_\infty~\text{(Lemma \ref{lemma:soft_bellman_operator_lemma})}.
    \end{align*}
\end{proof}

\newpage
\section{ Implementation Details and Additional Experiments } \label{append:impement_details}
\subsection{Traffic Light Control Environment}\label{subappend:environment}

We consider the traffic light control simulation environment \cite{sumorl,alegre21traffic1}, illustrated in Fig. \ref{fig:traffic_overview}. The intersection comprises four road directions (North, South, East, West), each consisting of four inbound and four outbound lanes. We configured the traffic flow to be asymmetric, with a fourfold higher influx of cars from the North and South compared to those from the East and West. We generated a corresponding route file using code provided by \citet{sumorl}. There are four available traffic light phases: (i) Straight and Turn Right from North-South, (ii) Turn Left from North-South, (iii) Straight and Turn Right from East-West, and (iv) Turn Left from East-West. 

At each time step, the agent receives a thirty-seven-dimensional state containing a one-hot vector indicating the current traffic light phase, the number of vehicles for each incoming lane, and the number of vehicles with a speed of less than 0.1 meter/second for each lane. The initial state is a one-hot vector with the first element one. Given the current state, the traffic light controller selects the next traffic light phase as its action. The simulation time between actions is 30 seconds, and each episode lasts for 9000 seconds, equivalent to 300 timesteps. If the current phase and the next phase (the current action) are different, the last 4 seconds of the 30-second interval transition to the yellow light phase to prevent collisions among vehicles. The reward at each timestep is a four-dimensional vector, with each dimension representing a quantity proportional to the negative of the total waiting time for cars on each road. The total number of timesteps in the simulation is set to 100,000.

\subsection{Implementation in the Traffic Environment}\label{subappend:implement}

We modified the implementation code of sumo-rl \cite{sumorl}, which primarily relies on stable-baselines3 \cite{raffin21stablebaselines}, a widely used reinforcement learning framework built on PyTorch \cite{paszke19pytorch}. For comparison with our value-based method, we consider the following value-based baselines: (i) Utilitarian, which is a standard Deep Q-learning (DQN) \cite{mnih13} using averaged rewards $\frac{1}{K} \sum_{k=1}^K r^{(k)}$, and (ii) Fair Min-DQN (MDQN), an extension of the Fair-DQN concept \cite{siddique20fair} to the max-min fair metric maximizing $\mathbb{E}[\min_{1 \leq k \leq K} \sum_t \gamma^t r^{(k)}_t]$. 

For both the proposed method and the baselines, we set $\gamma=0.99$ and the buffer size $|\mathcal{M}|=50,000$. All three methods employ a Q-network with an input dimension of 37 (state dimension), two hidden layers of size 64, and two ReLU activations after each hidden layer. The output layer size for the proposed method and Utilitarian is 4, corresponding to the action size. For MDQN, the output layer size is 16 ($4 \times 4$), representing the action size multiplied by the reward dimension size. We utilize the Adam optimizer \cite{kingma14adam} to optimize the loss function, with a learning rate of 0.001 and minibatch size 32. For the baselines, we use $\epsilon$-greedy exploration with linear decay from $\epsilon=1.0$ to 0.01 for the initial 10,000 timesteps. The interval of each target network is 500 timesteps.

The proposed method adopts the Soft Q-learning (SQL) conducted as follows given $w \in \Delta^K$:
\begin{equation}\label{eq:SQL_update}
    \min_\phi \frac{1}{|\mathcal{B}|} \sum_{(s,a,r,s') \in \mathcal{B} \subset \mathcal{M}} \bigg( \sum_{k=1}^K w_k r^{(k)}(s,a) + \gamma \alpha \log \sum_{a'} \exp \left( \frac{\hat{Q}_{\bar{\phi}}(s',a')}{\alpha} \right) - \hat{Q}_{\phi}(s,a) \bigg)^2
\end{equation}
where the soft Q-network $\hat{Q}$ is parameterized by $\phi$, $\bar{\phi}$ is the target parameter, and $\mathcal{B}$ is a minibatch. With $\hat{Q}_{w^m, \text{main}} = \hat{Q}_{\phi^m}$ and weight $w^m$ at the $m$-th step, SQL update is performed with $\alpha=0.1$ throughout all timesteps, followed by soft target update of ratio $\tau=0.001$ in $\bar{\phi} \leftarrow \tau \phi^{m+1} + (1-\tau) \bar{\phi}$. We use an exploration strategy for the current policy $\text{softmax}_a  \{ \hat{Q}_{\phi}(s,a) / \alpha_{\text{act}}  \}$ with linear decay from $\alpha_{\text{act}}=5.0$ to $0.1$ for the initial 10,000 timesteps. The weight vector $w$ is uniformly initialized across dimensions (Line 2 in Algorithm \ref{alg:maxent_maxmin}) and kept fixed for the first 50 timesteps, with one gradient step of \eqref{eq:SQL_update} per timestep (Line 3).

After 50 timesteps, we generate $N=20$ perturbed weights with $\mu=0.01$ (Line 6). Each $\hat{Q}_{w^m, \text{copy}, n}$ is updated using soft Q-learning with $w^m + \mu u^{m}_n$, employing common samples from $\mathcal{M}$ of size 32 (Lines 7-8). The target soft Q-network for each $\hat{Q}_{w^m, \text{copy}, n}$ is a copy of the current main target soft Q-network. As mentioned in Section \ref{sec:model_free}, we perform one step of gradient update for SQL for each copy with a learning rate of 0.001. Thus, the overall complexity of the proposed algorithm is at the level of SQL with slight increase due to linear regression at each step.

After the linear regression (Lines 9-11), we update the current weight $w^m$ using projected gradient descent, employing the technique from \cite{DBLP:journals/corr/WangC13a} (Line 12). The initial learning rate of the weight $w$ is set to $l_0 = 0.01$, and we employ inverse square root scheduling \cite{nesterov2017random} (Line 13). For the main Q-network update with the updated weight $w^{m+1}$, we perform 3 gradient steps per timestep to incorporate the new weight information (Line 14).

In MDQN, a vector-valued Q-network $Q_\theta(s,a) \in \mathbb{R}^K$ parameterized by $\theta$ is trained by $\min_\theta \mathbb{E}_{(s,a,r,s') \sim \mathcal{M}} \bigg[ \| r + \gamma \bar{Q}(s', \arg \max_{a'} ~ \min_{1 \leq k \leq K}[r^{(k)} + \gamma \bar{Q}^{(k)}(s',a')]  ) - Q_\theta(s,a)   \|^2  \bigg]$ where $\bar{Q}(s', a'  ) \in \mathbb{R}^K$ is a vector-valued target function. Here, the minimum of $r + \gamma \bar{Q}(s', \arg \max_{a'} ~ \min_{1 \leq k \leq K}[r^{(k)} + \gamma \bar{Q}^{(k)}(s',a')]  )$ over $K$ dimension is $\max_{a'} ~ \min_{1 \leq k \leq K}[r^{(k)} + \gamma \bar{Q}^{(k)}(s',a')]$. If $Q_\theta(s,a)$ approaches $r + \gamma \bar{Q}(s', \arg \max_{a'} ~ \min_{1 \leq k \leq K}[r^{(k)} + \gamma \bar{Q}^{(k)}(s',a')]  )$, then $\min_{1 \leq k \leq K} Q^{(k)}_\theta(s,a)$ approaches $\max_{a'} ~ \min_{1 \leq k \leq K}[r^{(k)} + \gamma \bar{Q}^{(k)}(s',a')]$. This implies that MDQN aims to maximize $\mathbb{E}_{(s,a)}[\min_{1 \leq k \leq K} Q^{(k)}_\theta(s,a)]$. Action selection is performed by $\arg \max_a \min_{1 \leq k \leq K} Q^{(k)}_\theta(s,a), \forall s$. Note that MDQN is reduced to the standard DQN with scalar reward when we set $K=1$. MDQN is related to optimizing $\mathbb{E}_\pi \left[ \min_k ( \sum_{t=0}^\infty \gamma^t {r}^{(k)}_t ) \right]$, rather than $\min(J(\pi)) = \min_k ( \mathbb{E}_\pi \left[ \sum_{t=0}^\infty \gamma^t {r}^{(k)}_t \right] )$. In contrast, we propose a value-based method that explicitly optimizes $\min (J(\pi))$. 

We used a hardware of Intel Core i9-10900X CPU @ 3.70GHz.

\subsection{Additional Experiments in Species Conservation}\label{subappend:additional_experiments}

We conducted additional experiments to further support  our  method.  We considered Species Conservation \cite{siddique23fair}, another widely used MORL environment. The agent goal is to take appropriate actions to balance the population of two species: the endangered sea otters and their prey, and  the elements of two-dimensional reward vector represent  quantities of the current predators (sea otters) and prey. We ran algorithms for 100,000 timesteps in this environment, as in the other two environments in Section \ref{subsec:performance}, and the metric is calculated over the 32 most recent episodes.

As shown in Table \ref{tab:Experiment_conservative_species_common}, the proposed method demonstrates superior max-min performance and achieves the most balanced outcomes. The return vector of conventional MDQN is Pareto-dominated by that of our algorithm, and our approach outperforms Utilitarian in terms of max-min fairness. Note that the Utilitarian approach, i.e., sum return maximization, yields extreme unbalance between Returns 1 and 2. We used tanh activation for our policy network.

\begin{table}[ht]
\centering
\begin{tabular}{|c|c|c|c|}
\hline
 & Return 1 & Return 2  & Min Return \\
\hline
Proposed & 27 & 38 & \textbf{27} \\
MDQN & 22 & 29 & 22 \\
Utilitarian & 4 & 87 & 4 \\
\hline
\end{tabular}
\caption{Performance in Species Conservation environment. }
\label{tab:Experiment_conservative_species_common}
\end{table}

\newpage
\subsection{Additional Analysis on Computation}\label{subappend:additional_ablation}

Our model-free algorithm does not increase complexity severely from existing soft value iterations.  As seen, our algorithm is composed of (a) weight $w$ update and  (b) soft Q value update with given $w$.  Step (b) is simply the conventional soft Q value update. Step (a) can be implemented efficiently by performing only \textbf{one} step of gradient update for Soft Q-learning  for each copy $\hat{Q}_{w^m, \text{copy}, n}$ in Line 8 of Algorithm \ref{alg:maxent_maxmin}, using common samples for updating each copy with Adam optimizer \cite{kingma14adam} in PyTorch, a common deep learning library.  Note that $N=20$ copies are sufficient as seen in Fig. \ref{fig:ablation_bar}.

We compared the runtime of our algorithm with that of simple soft Q-value iteration without the $w$ weight learning part. We considered two environments: the traffic control environment discussed in the paper and species conservation \cite{siddique23fair}, a newly introduced environment elaborated below in the more experimental results part. These computations were conducted on hardware equipped with an Intel Core i9-10900X CPU @ 3.70GHz. Our algorithm utilizes $N=20$ copies. As seen in Table \ref{tab:Experiment_runtime_complexity}, the runtime ratio is much lower than the value $N=20$ for both environments. In the case of traffic control, the increase in runtime is not significant. 

In the traffic control environment, we also computed  the average elapsed time per linear regression step, averaging over 500 steps. As shown in Table \ref{tab:regression}, the computation of linear regression scales efficiently for large values of $N$.

\begin{table}[ht]
\centering
\begin{tabular}{|c|c|c|c|}
\hline
  & Proposed algorithm & Soft value iteration without  weight learning  & Ratio \\
\hline
Species conservation & 6.7 & 1.3 & 5.1 \\
Traffic control & 65.4 & 60.3 & 1.1 \\
\hline
\end{tabular}
\caption{Average total runtime per episode in seconds. Each episode consists of 300 timesteps.}
\label{tab:Experiment_runtime_complexity}
\end{table}

\begin{table}[ht]
    \centering
    \begin{tabularx}{\textwidth}{X|XXXX}
        \toprule
        \textbf{$N$} & \textbf{5} & \textbf{10} & \textbf{20} & \textbf{30} \\
        \midrule
        Elapsed time (s) & $1.18 \times 10^{-4}$ & $1.22 \times 10^{-4}$ & $1.26 \times 10^{-4}$ & $1.35 \times 10^{-4}$ \\
        \bottomrule
    \end{tabularx}
    \caption{Elapsed time per one linear regression step in traffic control environment.}
    \label{tab:regression}
\end{table}

\newpage
\section{ Glossary }
\begin{table}[!ht]
	\centering
	\begin{adjustbox}{width=\textwidth}
		\begin{tabular}{ll}
			% \begin{tabular}{p{2.2cm} p{0.7cm} p{0.7cm} p{0.7cm} p{0.7cm}}
			\toprule
			%\multicolumn{7}{c}{}                   \\
			%\cmidrule{2-7}
			Notations & Descriptions   \\
			%\midrule
			\hline
			$S$ & State space \\
			$A$ & Action space  \\
			$P$ & Transition dynamics  \\
			$\mu_0$  & Initial state distribution \\
			$r$ & Reward vector in MOMDP  \\
                $K$ & Dimension of reward vector \\
            $r^{(k)}$ & $k$-th coordinate of vector reward $r$, $1\le k\le K$ \\
			$\gamma$ & Discount factor \\
			$p$ & Number of states, i.e. $|S|$\\
			$q$  & Number of actions, i.e. $|A|$   \\
			$\pi,\Pi$ & Policy, policy space \\
			$J(\pi)$ & Value vector under policy $\pi$ in MOMDP \\
            $J_k(\pi)$ & $k$-th coordinate of value vector under policy $\pi$ in MOMDP, $1\le k\le K$ \\
			$\Delta^K$ & $(K-1)$-Simplex, i.e., $\{w\in\mathbb{R}^K|\sum_{k=1}^K w_k =1, w_k\ge0,\forall 1\le k\le K\}$ \\
            $d(s,a)$ & State-action visitation frequency \\
			$\pi^d(a|s)$ & Stationary policy induced by $d$  \\
			$w^{op}_{LP},v^{op}_{LP}$ & Optimal solution of $\mathbf{P1}$  \\
                $w^*$ & Optimal solution of $\mathbf{P2}$ \\
			$T^*_w$ & Bellman optimality operator or in (single objective) MDP $<S,A,P,\mu_0,\sum_{k=1}^K w_k r^{(k)},\gamma>$   \\
            $\mathcal{T}^{soft,*}_w$ & Soft Bellman optimality operator in (single objective) MDP $<S,A,P,\mu_0,\sum_{k=1}^K w_k r^{(k)},\gamma>$ \\
			$v^*_w\in\mathbb{R}^{|S|}$ & Fixed point of $T^*_w$. Also, optimal state value function of (single objective) MDP $<S,A,P,\mu_0,\sum_{k=1}^K w_k r^{(k)},\gamma>$\\
                $v^{soft,*}_w\in\mathbb{R}^{|S|}$ & Fixed point of $\mathcal{T}^{soft,*}_w$.  Also, optimal soft value function of (single objective) MDP $<S,A,P,\mu_0,\sum_{k=1}^K w_k r^{(k)},\gamma>$\\
                $Q^*_w$ & Optimal action value function of (single objective) MDP $<S,A,P,\mu_0,\sum_{k=1}^K w_k r^{(k)},\gamma>$ \\
                $Q^{soft,*}_w$ & Optimal soft action value function of (single objective) MDP $<S,A,P,\mu_0,\sum_{k=1}^K w_k r^{(k)},\gamma>$ \\
			$\mathcal{L}(w)$ & Optimal value function averaged by initial states, i.e., $\sum_{s}\mu_0(s)v^*_w(s)$\\
                $\mathcal{L}^{soft}(w)$ & Optimal soft value function averaged by initial states, i.e., $\sum_{s}\mu_0(s)v^{soft,*}_w(s)$ \\
                $N$ & Number of perturbations\\
			\bottomrule
		\end{tabular}
	\end{adjustbox}
	\caption{ Used notations in the main paper  }
	\label{tab:notations}
\end{table}

%%%%%%%%%%%%%%%%%%%%%%%%%%%%%%%%%%%%%%%%%%%%%%%%%%%%%%%%%%%%%%%%%%%%%%%%%%%%%%%
%%%%%%%%%%%%%%%%%%%%%%%%%%%%%%%%%%%%%%%%%%%%%%%%%%%%%%%%%%%%%%%%%%%%%%%%%%%%%%%

\end{document}